\documentclass{article}

\usepackage{wrapfig}

 \usepackage[preprint]{neurips_2026}

\usepackage[utf8]{inputenc} % allow utf-8 input
\usepackage[T1]{fontenc}    % use 8-bit T1 fonts
\usepackage{hyperref}       % hyperlinks
\usepackage{url}            % simple URL typesetting
\usepackage{booktabs}       % professional-quality tables
\usepackage{amsfonts}       % blackboard math symbols
\usepackage{nicefrac}       % compact symbols for 1/2, etc.
\usepackage{microtype}      % microtypography
\usepackage{xcolor}         % colors

\usepackage{makecell}
\usepackage[table]{xcolor}
\usepackage{array}
\usepackage{arydshln}

\usepackage{microtype}
\usepackage{graphicx}
\usepackage{subcaption}
\usepackage{booktabs}
\usepackage{hyperref}

\usepackage{amsmath}
\usepackage{amssymb}
\usepackage{mathtools}
\usepackage{amsthm}
\usepackage{algorithm}
\usepackage{algorithmic}
\usepackage[capitalize,noabbrev]{cleveref}
\usepackage[textsize=tiny]{todonotes}
\usepackage{multirow}
\usepackage{neurips_2026} 

\theoremstyle{plain}
\newtheorem{theorem}{Theorem}[section]
\newtheorem{lemma}[theorem]{Lemma}
\newtheorem{proposition}[theorem]{Proposition}

\theoremstyle{definition}
\newtheorem{definition}[theorem]{Definition}

\title{FragileFlow: Spectral Control of Correct-but-Fragile Predictions for Foundation Model Robustness}

\author{%
  Zhuoyun Li, Boxuan Wang, Jinwei Hu, Xiaowei Huang, Yi Dong \\
  School of Computer Science and Informatics, University of Liverpool, UK
}

\begin{document}

\maketitle

\begin{abstract}
Robust adaptation of LLMs and VLMs is often evaluated by average accuracy or average consistency under perturbations.
However, these averages can hide a structured failure mode: a prediction may remain correct while probability mass already flows from particular true classes toward systematic wrong competitors near the decision boundary.
In this paper, we formalize this phenomenon as margin-aware error flow and introduce \emph{FragileFlow}, a plug-in regularizer that uses a calibrated margin buffer to identify correct-but-fragile predictions and organize their off-class probability mass into a class-wise vulnerable-risk matrix.
Theoretically, we provide the first PAC-Bayes upper bound for this margin-aware error-flow object, showing how empirical spectral control yields a conservative route to deterministic worst-class robustness under a stability condition.
Experiments on multiple-choice LLM benchmarks and few-shot CLIP adaptation show that FragileFlow consistently improves the proposed theory-facing risk measures over matched baselines, yields perturbed worst-class accuracy gains in most settings, and preserves clean accuracy across comparisons.
\end{abstract}

% 有地方的话加在Conclusion之前discussion
% Conclusion 我们提出了什么方法 解决了什么问题 在什么实验被验证效果如何 future work

\section{Introduction}
\label{sec:introduction}

Foundation models have evolved from simple text generators into central components of high-stakes decision-making~\cite{Bommasani2021FoundationModels,NEURIPS2020_1457c0d6,radford2021learning,Hu_Dong_Sun_Huang_2026}. In these practical applications, robustness is an important requirement, as a minor visual distraction or a subtle linguistic perturbation can easily alter a critical recommendation~\cite{BENCHMARKS2021_335f5352,10.1145/3689217.3690621,mao2023understanding,hu2026lying,pmlr-v235-schlarmann24a}. Consequently, a reliable model must not only perform well on clean inputs but also remain stable under various perturbations.

% While existing robust adaptation methods, such as adversarial training and regularized fine-tuning, have made significant strides, the gaps stills remain. First, they predominantly optimize for average robustness, measuring whether overall accuracy or loss improves while largely ignoring the underlying mechanics of how a prediction degrades. This misses a fragile regime in which the predicted option remains correct and looks good on test set, yet perturbations already push substantial probability toward systematic wrong competitors. Even when the true option retains the highest score, perturbations can silently push substantial probability mass toward systematic wrong competitors. Second, many practical adaptation techniques are justified primarily by empirical gains. They lack rigorous theoretical generalization bounds tied to the specific fragile behaviors they aim to suppress. As a result, a model may appear stable on standard benchmarks while hiding structured vulnerabilities, leaving the worst-class risk unchecked.
Existing robust adaptation methods have made important progress toward this goal. 
Adversarial training exposes models to difficult perturbations, while regularized fine-tuning methods encourage local smoothness or consistency between clean and perturbed inputs~\cite{DBLP:journals/corr/GoodfellowSS14,aghajanyan2021better,madry2018towards,zhang2019theoretically,jiang2020smart,NEURIPS2020_1ef91c21}.
In LLMs, this is often implemented through adversarial, KL-based, noise-based, or trust-region regularization~\cite{Zhu2020FreeLB:,jiang2020smart,aghajanyan2021better}; in VLMs, robustness is commonly improved by perturbing visual inputs, tuning prompts, or adapting lightweight components~\cite{mao2023understanding,10.1007/978-3-031-72995-9_4,pmlr-v235-schlarmann24a,11095021}. 
However, gaps still remain. 
First, most objectives still optimize an average notion of robustness: they ask whether accuracy, loss, or consistency improves on average, but they do not track how probability mass moves among wrong options before the final prediction fails~\cite{xu2021robust,pmlr-v148-benz21a,tian2021analysis,10.1609/aaai.v37i12.26749}. 
This misses a fragile regime in which the predicted option remains correct, yet perturbations already push substantial probability toward systematic wrong competitors. 
Second, many practical robust adaptation methods are supported mainly by empirical improvements, while their objectives are not tied to a generalization bound for the specific fragile behavior they aim to reduce~\cite{neyshabur2018pac,lotfi2022pacbayes,NEURIPS2025_9408564a,jin2025enhancing}. 
As a result, a model may look stable under standard aggregate metrics while hiding structured vulnerabilities.

To address this empirical and theoretical gap, 
% we use finite-option LLM and VLM tasks as a controlled interface. Because the predictive distribution over finite labels is directly observable, it allows us to precisely track this vulnerable probability flow without the ambiguity of open-ended generation. Rather than merely asking whether a final answer flips, we track exactly where the confidence leaks. Based on this insight, we introduce FragileFlow, a lightweight, plug-in regularizer. FragileFlow constructs a margin-aware error-flow matrix under perturbation and actively suppresses its dominant spectral mode. The margin buffer specifically targets examples that are already wrong or precariously close to the decision boundary, while the spectral penalty disrupts coherent probability flow toward recurring distractors.
we study this correct-but-fragile regime through finite-option LLM and VLM tasks. 
This setting gives a controlled interface: the model's distribution over candidate options is directly observable, and each wrong option has a clear semantic meaning. 
Instead of asking only whether the final answer flips, we ask where the off-class probability goes, which true options are most vulnerable, and whether the resulting error pattern is scattered or structured. 
Based on this observation, we introduce \textbf{FragileFlow}, a lightweight plug-in regularizer for robust adaptation. 
FragileFlow constructs a margin-aware error-flow matrix under perturbation. 
The margin buffer focuses on examples that are already misclassified or still correct but close to the decision boundary, while the spectral penalty suppresses coherent probability flow from true options toward recurring wrong competitors. 
In this way, FragileFlow targets a failure pattern that average robustness objectives can easily miss.

Crucially, we provide the theoretical foundation that existing empirical methods lack. While PAC-Bayes analysis has been used for standard generalization and begun to extend to modern neural networks and language models~\cite{jin2025enhancing,nagarajan2018deterministic}, we establish, to our knowledge, the first PAC-Bayes spectral control framework tailored for margin-aware error flow. Our bound rigorously connects the empirical spectral error-flow term to the population's vulnerable worst-class risk. Under a stated logit-stability condition, this demonstrates that FragileFlow is not merely a heuristic add-on, but is mathematically aligned with deterministic worst-class robustness at test time, directly bridging the gap between empirical adaptation and provable generalization. Our main contributions are threefold:

\begin{itemize}
    \item We formalize the ``correct-but-fragile'' failure mode in robust adaptation, showing how models can silently leak probability mass toward systematic wrong competitors near the decision boundary before the final prediction fails.

    \item We introduce \emph{FragileFlow}, a margin-aware spectral plug-in regularizer that suppresses structured vulnerable probability flow and can be attached to existing adaptation objectives.

    \item We bridge the theoretical gap in robust adaptation by establishing the first PAC-Bayes spectral control route for this error-flow object. We further show its connection to deterministic worst-class risk and validate its effectiveness across both LLM and VLM settings.
\end{itemize}

\section{Methodology}
\label{sec:methodology}
\label{sec:prelim}

% To ..., we need to address five research questions:
% \textbf{RQ1:} What prediction interface allows us to observe option-wise probability movement?
% \textbf{RQ2:} Which risk object captures correct-but-fragile behavior under perturbation?
% \textbf{RQ3:} Can we obtain theoretical guarantees rather than relying solely on empirical performance?
% \textbf{RQ4:} How does this posterior risk relate to the deterministic model used at test time?
% \textbf{RQ5:} How can this theory be implemented during training?
% We answer each of these questions in the following sections.
% We now develop FragileFlow in the order shown in Fig.~\ref{fig:workflow}. 
% The method starts from a finite-option prediction interface, builds a margin-aware error-flow matrix under perturbation, derives a PAC-Bayes spectral control route for this risk object, connects the posterior risk to the deterministic model used at test time, and finally turns the empirical spectral term into a plug-in regularizer for robust adaptation.

\subsection{Prediction and Perturbation Formulation}
\label{sec:problem_formulation}
\label{sec:token_cls}
\label{sec:robust_setup}
We now develop FragileFlow in the order shown in Fig.~\ref{fig:workflow}. 
First, we set up the notation for finite-option prediction and perturbations.
Each input is assigned one label from a fixed
set of $K$ candidate options. We first define the prediction interface and the
perturbation notation for a fixed adapted model $\theta$. Randomized adaptation
parameters are introduced later in Section~\ref{sec:pac_bayes_control}.

\begin{definition}[Finite-option prediction]
Let $\mathcal{D}$ be a distribution over input--label pairs $(x,y)$ with
$y\in\{1,\ldots,K\}:=[K]$. Let
$S=\{(x_r,y_r)\}_{r=1}^{m}$ be a finite sample drawn from $\mathcal{D}$.
For a model $\theta$, each input $x$ induces a score vector
$s_\theta(x)\in\mathbb{R}^{K}$, where $s_\theta(x,k)$ is the score assigned to
option $k$. The induced option distribution is
% \begin{equation}
$p_\theta(k\mid x)
:=
\frac{\exp(s_\theta(x,k))}
{\sum_{k'=1}^{K}\exp(s_\theta(x,k'))}$.
% \label{eq:score_prob_def}
% \end{equation}
The deterministic predictor used for evaluation is
% \begin{equation}
$\hat y^{\mathrm{det}}_\theta(x)
:=
\arg\max_{k\in[K]}s_\theta(x,k)$.
% \label{eq:det_predictor}
% \end{equation}
\end{definition}

\begin{definition}[Option margin]
For an example $(x,y)$, the option margin is
% \begin{equation}
$\Delta_\theta(x,y)
:=
s_\theta(x,y)-\max_{k\neq y}s_\theta(x,k)$.
% \label{eq:option_margin}
% \end{equation}
A positive margin means that the correct option is selected by the deterministic
predictor. 
A negative margin means the predictor selects a wrong option.
A small positive margin means that the prediction is still correct,
but a wrong option is close to overtaking the true one.
\end{definition}

\begin{definition}[Perturbed distribution and sample]
For each input $x$, let $\mathcal{U}(x)$ be the allowed perturbation set. A
perturbation rule $\Pi(\cdot\mid x,y)$ is supported on $\mathcal{U}(x)$ and may
represent either random or adversarially selected perturbations. The perturbed
distribution $\mathcal{D}'$ is induced by drawing $(x,y)\sim\mathcal{D}$ and then
drawing $x'\sim\Pi(\cdot\mid x,y)$. Given the finite sample $S$, the corresponding
perturbed sample is
% \begin{equation}
$S'
:=
\{(x'_r,y_r)\}_{r=1}^{m},
\quad
x'_r\sim\Pi(\cdot\mid x_r,y_r)$.
% \label{eq:perturbed_sample}
% \end{equation}
\end{definition}

\begin{figure}
    \centering
    \includegraphics[width=1\linewidth]{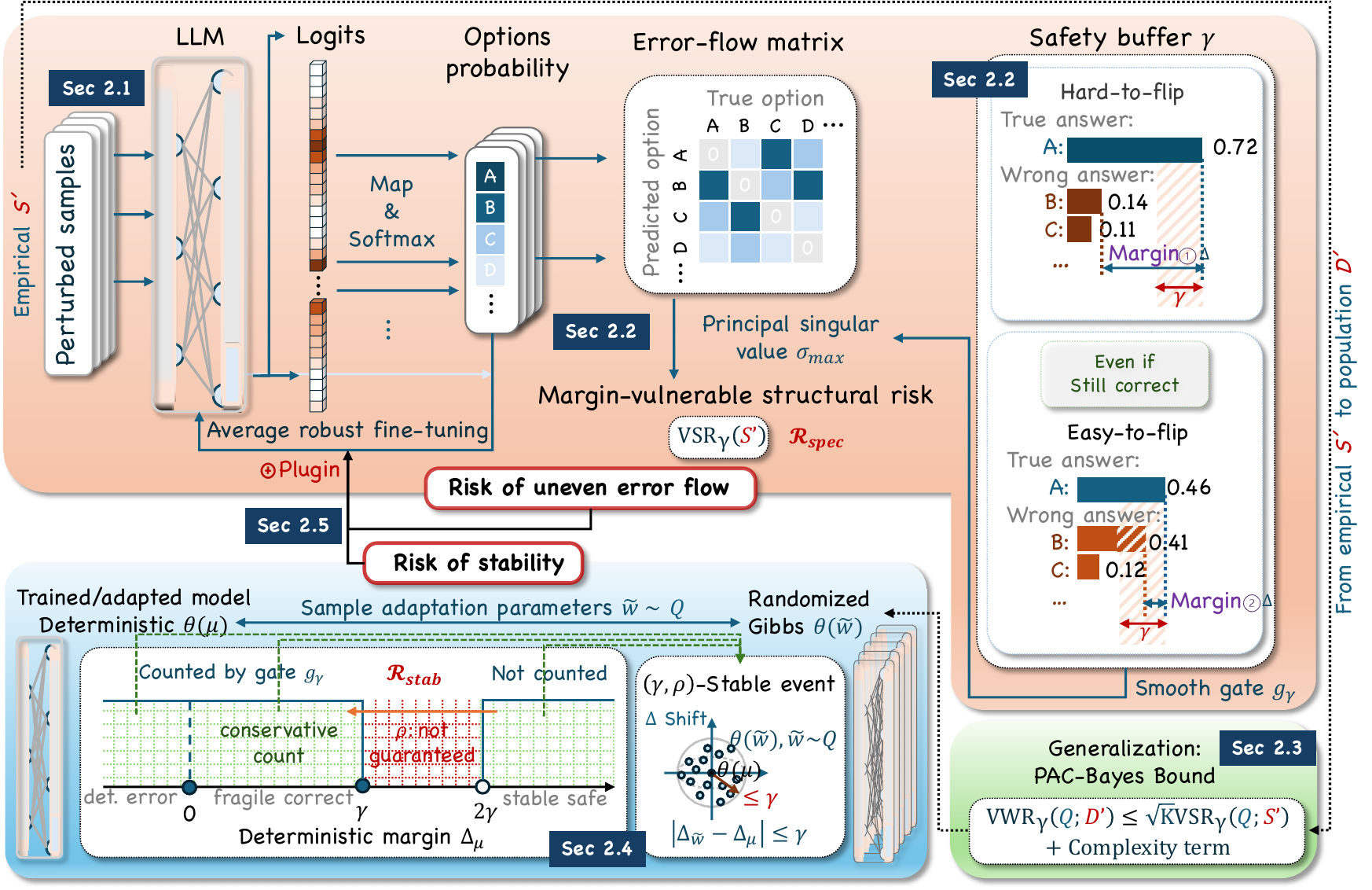}
\caption{\textbf{Overview of FragileFlow.}
The figure summarizes the pipeline from finite-option prediction to margin-aware error flow, PAC-Bayes control, deterministic stability, and the final plug-in objective. Section tags indicate where each component is defined.}
\label{fig:workflow}
    \label{fig:mianflowchart}
\end{figure}

\subsection{Margin-aware Error Flow}
\label{sec:error_flow}

Perturbations can change more than the final predicted option. They can also
change where the model places probability among the wrong options. We first
record this off-option allocation with an error-flow matrix.

\begin{definition}[Ungated error-flow matrix]
For a fixed model $\theta$, the ungated population error-flow matrix under
$\mathcal{D}'$ is
\begin{equation}
(M_{\theta}^{\mathcal{D}'})_{ij}
:=
\mathbb{E}_{(x',y)\sim\mathcal{D}'}
\left[p_\theta(i\mid x')\mid y=j\right],
\quad i\neq j,
\qquad
(M_{\theta}^{\mathcal{D}'})_{jj}:=0 .
\label{eq:ungated_flow_matrix}
\end{equation}
Rows index the receiving option $i$, and columns index the ground-truth option
$j$. The zero diagonal removes probability mass assigned to the correct option.
% For a perturbed sample $S'$, with $m_j:=|\{r:y_r=j\}|$, the empirical counterpart is
% \begin{equation}
% (M_{\theta}^{S'})_{ij}
% :=
% \frac{1}{m_j}
% \sum_{r:y_r=j}p_\theta(i\mid x'_r),
% \quad i\neq j,
% \qquad
% (M_{\theta}^{S'})_{jj}:=0 .
% \label{eq:empirical_ungated_flow_matrix}
% \end{equation}
This matrix tells us where wrong-option probability goes, but not how dangerous
that probability is. A wrong option receiving extra probability is less concerning
when the true option is far ahead. It is more concerning when the true option is
only slightly ahead, because a small additional shift can flip the prediction.
\end{definition}

\begin{definition}[Margin gate]
Given a safety buffer $\gamma\geq 0$ and a temperature $\kappa>0$, the smooth
margin gate is
\begin{equation}
g_{\gamma,\kappa}^{\theta}(x',j)
:=
\sigma\left(\frac{\gamma-\Delta_\theta(x',j)}{\kappa}\right)
\in(0,1),
\label{eq:gate_sigmoid_prelim}
\end{equation}
where $\sigma(\cdot)$ is the logistic sigmoid. This gate is a differentiable
version of $\mathbf{1}\{\Delta_\theta(x',j)\leq\gamma\}$. It is large when the
perturbed example is already misclassified or lies inside the safety buffer, and
small when the true option has a clear margin.
Intuitively, $\gamma$ draws a band around the decision boundary. Examples inside
this band are not all wrong, but they are easy to flip. 
We therefore weight the off-option probability in Eq.~\ref{eq:ungated_flow_matrix}
by this margin gate.
\end{definition}

\begin{definition}[Margin-aware error-flow matrix]
The margin-aware population error-flow matrix is
\begin{equation}
(M_{\theta}^{\mathcal{D}',\gamma})_{ij}
:=
\mathbb{E}_{(x',y)\sim\mathcal{D}'}
\left[
g_{\gamma,\kappa}^{\theta}(x',j)
p_\theta(i\mid x')
\mid y=j
\right],
\quad i\neq j,
\qquad
(M_{\theta}^{\mathcal{D}',\gamma})_{jj}:=0 .
\label{eq:M_margin_population}
\end{equation}

% For the sample $S'$, with $m_j:=|\{r:y_r=j\}|$, its empirical form is
% $(M_{\theta}^{S',\gamma})_{ij}:=
% m_j^{-1}\sum_{r:y_r=j}
% g_{\gamma,\kappa}^{\theta}(x'_r,j)p_\theta(i\mid x'_r)$ for $i\neq j$,
% and $(M_{\theta}^{S',\gamma})_{jj}:=0$.
Larger values of $\gamma$ include a wider safety band, so the matrix becomes more
conservative. We quantify the risk of the phenomenon:
% : it counts not only current mistakes, but also correct predictions
% that are close to becoming mistakes.
\end{definition}

\begin{definition}[Vulnerable risks]
For a fixed model $\theta$, the vulnerable worst-option risk is
% \begin{equation}
$\mathrm{VWR}_{\gamma}(\theta;\mathcal{D}')
:=
\|M_{\theta}^{\mathcal{D}',\gamma}\|_1$.
% \label{eq:vwr_def}
% \end{equation}
The vulnerable spectral risk is
% \begin{equation}
$\mathrm{VSR}_{\gamma}(\theta;\mathcal{D}')
:=
\|M_{\theta}^{\mathcal{D}',\gamma}\|_2$.
% \label{eq:vsr_def}
% \end{equation}
\end{definition}
Here $\|\cdot\|_1$ denotes the induced matrix $1$-norm, i.e., the maximum column
sum, and $\|\cdot\|_2$ denotes the spectral norm.
The first quantity asks which ground-truth option leaks the most vulnerable
probability mass to wrong options. The second asks whether this leakage is
structured rather than scattered. For example, several true options may drift
toward the same wrong option, or a group of options may become mutually
confusable. The standard norm conversion gives
$\mathrm{VWR}_{\gamma}(\theta;\mathcal{D}')
\leq
\sqrt{K}\,\mathrm{VSR}_{\gamma}(\theta;\mathcal{D}')$
(see Appendix~\ref{app:proof_l1_spec}). Thus, controlling the spectral risk gives
a conservative way to reduce the worst-option vulnerable readout while also
penalizing coherent error-flow patterns.

\subsection{PAC-Bayes Spectral Control with Randomized Adaptation}
\label{sec:gibbs_pacbayes}
\label{sec:posterior_coords}
\label{sec:matrix_risk}
\label{sec:pacbayes_bound}
\label{sec:main_theorem}
\label{sec:pac_bayes_control}

The previous definitions describe the vulnerable error-flow object on the
perturbed population $\mathcal{D}'$. In practice, we only observe a
finite perturbed sample $S'$. For generalization, we need to relate the empirical matrix on
$S'$ to the population risk on $\mathcal{D}'$. We use PAC-Bayes analysis because
it gives a direct way to separate the observed empirical term from the complexity
of the adapted model.

\begin{definition}[Randomized adaptation]
\label{def:randomized_adaptation}
Let $w\in\mathbb{R}^{d_{\mathrm{train}}}$ denote the trainable adaptation
coordinates, and write $\theta(w):=\mathcal{T}(\theta_0,w)$, where $\theta_0$ is
fixed and $\mathcal{T}$ specifies how $w$ is inserted into the predictor. Let
$P$ be a prior over $w$ chosen before observing the sample, and let $Q$ be a
posterior after observing the sample. We write
$\mu:=\mathbb{E}_{\widetilde w\sim Q}[\widetilde w]$, and the deterministic
adapted model used for evaluation is $\theta(\mu)$.
\end{definition}

For the analysis, a sampled coordinate $\widetilde w\sim Q$ induces a model
$\theta(\widetilde w)$. We use the corresponding option distribution
$p_{\theta(\widetilde w)}(\cdot\mid x)$ to define a Gibbs predictor. This
randomized predictor is only an analysis device; the empirical results still use
the deterministic model $\theta(\mu)$. The connection between these two
predictors is handled in Section~\ref{sec:gibbs_to_det}.

\begin{definition}[Posterior-averaged vulnerable risks]
We fix the perturbation protocol before the PAC-Bayes analysis, so
$\mathcal{D}'$ and $S'$ are the population and empirical perturbed objects. The
remaining randomness comes from $\widetilde w\sim Q$. The posterior-averaged
margin-aware matrices are
$\bar M_{\mathcal{D}',\gamma}^{Q}
:=
\mathbb{E}_{\widetilde w\sim Q}
\left[M_{\theta(\widetilde w)}^{\mathcal{D}',\gamma}\right] \text{and }
\bar M_{S',\gamma}^{Q}
:=
\mathbb{E}_{\widetilde w\sim Q}
\left[M_{\theta(\widetilde w)}^{S',\gamma}\right]$.
The posterior vulnerable worst-option risk and spectral risk are
$\mathrm{VWR}_{\gamma}(Q;\mathcal{D}')
:=\|\bar M_{\mathcal{D}',\gamma}^{Q}\|_1$ and
$\mathrm{VSR}_{\gamma}(Q;\mathcal{D}')
:=\|\bar M_{\mathcal{D}',\gamma}^{Q}\|_2$, respectively.
We also write
$\mathrm{VSR}_{\gamma}(Q;S')
:=\|\bar M_{S',\gamma}^{Q}\|_2$
for the empirical posterior spectral risk.
% By the same norm conversion,
% $\mathrm{VWR}_{\gamma}(Q;\mathcal{D}')
% \leq
% \sqrt{K}\,\mathrm{VSR}_{\gamma}(Q;\mathcal{D}')$.
\end{definition}

The following theorem gives the empirical-to-population step. It shows that a
small empirical spectral error-flow signal on $S'$ controls the population
vulnerable worst-option risk, up to an adaptation-complexity term.

\begin{theorem}[PAC-Bayes control of margin-aware vulnerable risk]
\label{thm:pac_bayes_control}
\label{thm:pacbayes_spectral}
Fix $\gamma\geq 0$ and $\delta\in(0,1)$. Let $m_j$ be the number of option-$j$
samples used in the option-conditional empirical matrix, and let
$m_{\min}:=\min_{j\in[K]}m_j\geq 1$. Condition on the realized option counts
$\{m_j\}_{j=1}^{K}$ and on the perturbation protocol that induces
$\mathcal{D}'$ and $S'$. Then, with probability at least $1-\delta$ over the
draw of $S$ and the induced perturbed sample $S'$, for all posteriors $Q$
simultaneously,
\begin{equation}
\label{eq:pac}
    \mathrm{VWR}_{\gamma}(Q;\mathcal{D}')
\leq
\sqrt{K}\,\mathrm{VSR}_{\gamma}(Q;S')
+
2\sqrt{
\frac{
2K\left(
\mathrm{KL}(Q\|P)+2K\ln 9+\ln\frac{2}{\delta}
\right)}
{m_{\min}}
}.
\end{equation}
\end{theorem}

The theorem gives the control chain used by our method. The right-hand first term is the
empirical vulnerable spectral risk measured on $S'$, which is the quantity later
targeted by the plug-in regularizer. The second term is the price of adaptation:
it grows with $\mathrm{KL}(Q\|P)$ and decreases with the smallest option-wise
sample size $m_{\min}$. The buffer $\gamma$ changes which examples contribute to
the vulnerable matrix, but it does not add a separate complexity term. The proof
is given in Appendix~\ref{app:proof_pacbayes_spectral}.

\subsection{From Posterior Risk to Deterministic Worst-class Risk}
\label{sec:gibbs_to_det}
\label{sec:deterministic_bridge}

Theorem~\ref{thm:pac_bayes_control} controls a posterior-averaged risk, while
the model used at test time is the deterministic adapted model $\theta(\mu)$.
We now connect these two objects through a simple logit-stability condition.

\begin{definition}[Deterministic worst-class risk]
For a deterministic model $\theta$, the perturbed worst-class risk is
$\mathrm{WCR}^{\mathrm{det}}(\theta;\mathcal{D}')
:=
\max_{j\in[K]}
\Pr_{(x',y)\sim\mathcal{D}'}
(\hat y^{\mathrm{det}}_{\theta}(x')\neq j\mid y=j)$.
\end{definition}

This is the hard-error quantity used for downstream evaluation. It differs from
$\mathrm{VWR}_{\gamma}$, which measures gated off-option probability mass rather
than hard prediction errors.

\begin{definition}[Posterior logit stability]
For a posterior sample $\widetilde w\sim Q$, define the largest option-score
shift from the posterior mean model as
\begin{equation}
\Xi_Q(\mu,\widetilde w)
:=
\sup_{(x',y)\in\mathrm{supp}(\mathcal{D}')}
\max_{k\in[K]}
\left|
s_{\theta(\widetilde w)}(x',k)-s_{\theta(\mu)}(x',k)
\right|.
\label{eq:posterior_logit_shift}
\end{equation}
We say that the posterior mean is $(\gamma,\rho)$-stable if
$\Pr_{\widetilde w\sim Q}(\Xi_Q(\mu,\widetilde w)\leq\gamma/2)\geq 1-\rho$.
\end{definition}

The stability condition says that most posterior samples stay close to the mean
model in option-score space. On the stable event
$\Xi_Q(\mu,\widetilde w)\leq\gamma/2$, every option score changes by at most
$\gamma/2$, so the margin changes by at most $\gamma$. The same buffer $\gamma$
therefore has two roles: it defines the vulnerable band in the gate, and it also
absorbs bounded posterior score fluctuations around the mean model. As a result,
when the mean model makes a deterministic error on a perturbed example, every
stable posterior sample remains inside the $\gamma$-vulnerable region counted by
the gate. Only the unstable $\rho$-probability event can escape this accounting.
The detailed case analysis is shown in Appendix~\ref{app:stability_event_cases}.

\begin{proposition}[Deterministic bridge under logit stability]
\label{prop:deterministic_bridge}
Suppose the posterior mean is $(\gamma,\rho)$-stable. Also suppose that the gate
satisfies $g_{\gamma,\kappa}^{\theta(\widetilde w)}(x',j)\geq\eta$ whenever
$\Delta_{\theta(\widetilde w)}(x',j)\leq\gamma$, for some $\eta>0$. Then
\[
\mathrm{WCR}^{\mathrm{det}}(\theta(\mu);\mathcal{D}')
\leq
\eta^{-1}(1+e^{\gamma})
\mathrm{VWR}_{\gamma}(Q;\mathcal{D}')
+
\rho .
\]
\end{proposition}

% The factor $(1+e^{\gamma})$ comes from converting a margin event into a lower
% bound on off-option probability mass: if an option lies within margin $\gamma$,
% then the total probability assigned to wrong options is at least
% $(1+e^{\gamma})^{-1}$. The factor $\eta^{-1}$ accounts for the smooth gate, and
% $\rho$ accounts for posterior samples whose logits move too far from the mean.
The proof is given in Appendix~\ref{app:deterministic_bridge}-\ref{app:proof_margin_failure_to_vwr}.
Together, Theorem~\ref{thm:pac_bayes_control} and
Proposition~\ref{prop:deterministic_bridge} give the control route used in this
paper. The empirical spectral term controls the posterior vulnerable risk, and
logit stability transfers this control to the deterministic model evaluated at
test time. In the experiments, perturbed worst-class accuracy is the downstream
readout, while $\mathrm{VWR}_{\gamma}$ and $\mathrm{VSR}_{\gamma}$ are the
theory-facing quantities.

\subsection{Plug-in Spectral Safety Control}
\label{sec:method}
\label{sec:method_overview}
\label{sec:plugin_control}

The theory suggests a direct plug-in principle: reduce the empirical vulnerable
spectral risk, while encouraging local logit stability under input and coordinate
perturbations. We use the same trainable coordinate notation $w$ as in
Definition~\ref{def:randomized_adaptation}.

\begin{definition}[Plug-in objective]
Let $w$ denote the current trainable coordinates, and let
$\mathcal{L}_{\mathrm{base}}$ be a standard training objective. For a fixed
safety buffer $\gamma$, we optimize
\begin{equation}
\mathcal{L}_{\mathrm{total}}(w)
:=
\mathcal{L}_{\mathrm{base}}(w)
+
\alpha \mathcal{R}_{\mathrm{spec}}(w;\gamma)
+
\beta \mathcal{R}_{\mathrm{stab}}(w),
\label{eq:plugin_objective}
\end{equation}
where $\alpha,\beta\geq 0$ are validation-selected weights.
Here $\mathcal{R}_{\mathrm{spec}}$ targets the empirical spectral term in
Theorem~\ref{thm:pac_bayes_control}, while $\mathcal{R}_{\mathrm{stab}}$ is an
output-level proxy for the logit-stability condition in
Proposition~\ref{prop:deterministic_bridge}.
\end{definition}

\begin{definition}[Mini-batch spectral regularizer]
Given a paired mini-batch, 
% let $B\subseteq\{1,\ldots,m\}$ be its index set and
let $B_j:=\{r\in B:y_r=j\}$. For every option $j$ with $|B_j|>0$, we form the
batch margin-aware estimator
% \begin{equation}
$(\widetilde M_{w}^{B,\gamma})_{ij}
:=
\frac{1}{|B_j|}
\sum_{r\in B_j}
g_{\gamma,\kappa}^{\theta(w)}(x'_r,j)
p_{\theta(w)}(i\mid x'_r),
 i\neq j,
(\widetilde M_{w}^{B,\gamma})_{jj}:=0$.
% \end{equation}
Then,
\begin{equation}
\mathcal{R}_{\mathrm{spec}}(w;\gamma)
:=
\widehat{\mathrm{VSR}}_\gamma(w;B)
:=
\left\|
\widetilde M_{w}^{B,\gamma}
\right\|_2 .
\label{eq:soft_M_batch}
% \label{eq:R_spec_def}
\end{equation}
\end{definition}
If option $j$ is absent from the mini-batch, its column is skipped for that
update. 
This term suppresses the dominant structured mode of vulnerable probability
flow; implementation details are given in Appendix~\ref{app:power_iter_grad}.

\begin{definition}[Local stability regularizer]
Let $\widetilde Q_w$ be a local perturbation distribution centered at $w$, and
draw $\tilde w\sim\widetilde Q_w$. For the same paired mini-batch $B$, define
\begin{equation}
\mathcal{R}_{\mathrm{stab}}(w)
:=
\mathbb{E}_{\tilde w\sim \widetilde Q_w}
\left[
\frac{1}{|B|}
\sum_{r\in B}
\mathrm{KL}
\left(
\operatorname{sg}\!\left[p_{\theta(w)}(\cdot\mid x_r)\right]
\,\middle\|\,
p_{\theta(\tilde w)}(\cdot\mid x'_r)
\right)
\right],
\label{eq:R_stab_def}
\end{equation}
where $\operatorname{sg}[\cdot]$ denotes stop-gradient.
This term penalizes prediction changes caused jointly by input perturbation and
local coordinate noise. Its coordinate-noise component discourages large centered
logit and margin shifts, making it a practical proxy for reducing the
stability-failure probability $\rho$ in
Proposition~\ref{prop:deterministic_bridge}. See
Appendix~\ref{app:stab_local_interpretation}.
\end{definition}
\section{Experiments}
\label{sec:exp}

\subsection{Experimental setup}
\label{sec:exp_setup}\label{sec:protocol}

To validate the theoretical control route and test its practical effect, we
evaluate FragileFlow under a unified finite-option prediction protocol across
two perturbation channels. In LLM tasks, perturbations modify the question text;
in VLM tasks, perturbations modify the image input. In both cases, the model
produces a distribution over candidate options, so the same margin-aware
error-flow matrix, calibrated buffer, and worst-class readouts can be used. This
lets us test whether the proposed risk-control mechanism transfers across
text-side and image-side robustness settings.

\noindent\textbf{Models and tasks.}
For LLMs, we evaluate Qwen2.5-0.5B-Instruct and Qwen2.5-1.5B-Instruct
\cite{qwen2025qwen25,qwen2024qwen2515binstruct}, and
Mistral-7B-Instruct-v0.2
\cite{jiang2023mistral7b,mistralai2023mistral7binstructv02} on
ARC-Challenge \cite{clark2018think} and CommonsenseQA
\cite{talmor2019commonsenseqa}. These multiple-choice tasks expose the full
verbalizer distribution, so the class-conditional error-flow matrix can be formed
directly. We test three task-preserving perturbations: typo noise, distractor
insertion, and format rewriting, following text robustness and behavioral testing
protocols \cite{morris2020textattack,ribeiro2020checklist}. For VLMs, we evaluate
CLIP ViT-B/32 \cite{radford2021learning} following \cite{pmlr-v235-schlarmann24a}
with LoRA adaptation \cite{hu2022lora}
on DTD \cite{cimpoi2014describing}, OxfordPets \cite{parkhi2012cats}, and
Caltech101 \cite{fei2007learning}. Robustness is measured on PGD-perturbed test
images \cite{madry2018towards}, following the adversarial LoRA adaptation
protocol of \cite{ghiasvand2025fewshotadversariallowrankfinetuning}.

\noindent\textbf{Baselines and plug-in placement.}
For LLMs, we attach FragileFlow to cross-entropy training (CE) with augmentation \cite{morris2020textattack}, R3F
\cite{aghajanyan2021better}, and SMART \cite{jiang2020smart}, covering
data-level augmentation, randomized embedding smoothness, and local smoothness
regularization. Plain CE is included as an unpaired reference. For VLMs, we apply
the plug-in to the \textit{inner} PGD step, the \textit{outer} model update, or \textit{both}, separating
whether the regularizer shapes adversarial-example construction, regularizes the
adapted model, or combines the two effects.

\noindent\textbf{Metrics and protocol.}
We report clean accuracy, perturbed worst-class accuracy,
$\widehat{\mathrm{VWR}}_{\gamma}$, and $\widehat{\mathrm{VSR}}_{\gamma}$; the
VLM table additionally reports perturbed average accuracy and clean worst-class
accuracy. Higher accuracy is better, while lower
$\widehat{\mathrm{VWR}}_{\gamma}$ and $\widehat{\mathrm{VSR}}_{\gamma}$ indicate
less margin-vulnerable probability flow. All paired comparisons use the same
data split, perturbation, random seeds, and calibrated
buffer.

% \noindent\textbf{Compute resources.} All experiments are conducted in Python\textsuperscript{\textregistered} on a machine equipped with an AMD EPYC\textsuperscript{\textregistered} 7452 32-Core Processor, 128GB of RAM, and one A100 GPU with 40GB of VRAM.

\subsection{Safety-buffer calibration}
\label{sec:exp_gamma}

The buffer $\gamma$ determines which perturbed examples are counted as
margin-vulnerable by the error-flow matrix. Since raw logit margins vary across
models and tasks, we do not fix $\gamma$ as a global numeric value. Instead, for
each setting, we compute margins on a held-out perturbed validation split and set
$\gamma=\gamma_q$, the $q$-th quantile of the validation-margin distribution. We
sweep $q\in\{0.10,0.25,0.50\}$ and compare each base learner with its plug-in
counterpart under the same calibrated buffer.

Figure~\ref{fig:gamma_calibration} shows the calibration behavior on the LLM
settings. The last two columns report $\widehat{\mathrm{VWR}}_{\gamma}$ and
$\sigma_{\max}$, where
$\sigma_{\max}=\widehat{\mathrm{VSR}}_{\gamma}$ represents the empirical spectral norm
of the margin-aware error-flow matrix. As $q$ increases, the buffer covers a
wider vulnerable region, so more perturbed examples and off-option probability
mass are counted; the measured vulnerable risks therefore increase by design.
The meaningful comparison is within the same $q$, where the plug-in consistently
moves the operating point toward lower vulnerable flow. The intermediate choice
$q=0.25$ gives the most stable trade-off: it preserves clean accuracy, maintains
or improves perturbed worst-class accuracy, and reduces the two margin-aware risk
measures. We therefore use $\gamma_{25}$, the 25th percentile of held-out
perturbed validation margins, in all remaining experiments.

% Main calibration figure that justifies the choice gamma = 0.25 used in the rest of the paper.
\begin{figure}[h]
    \centering
    \includegraphics[width=0.92\textwidth]{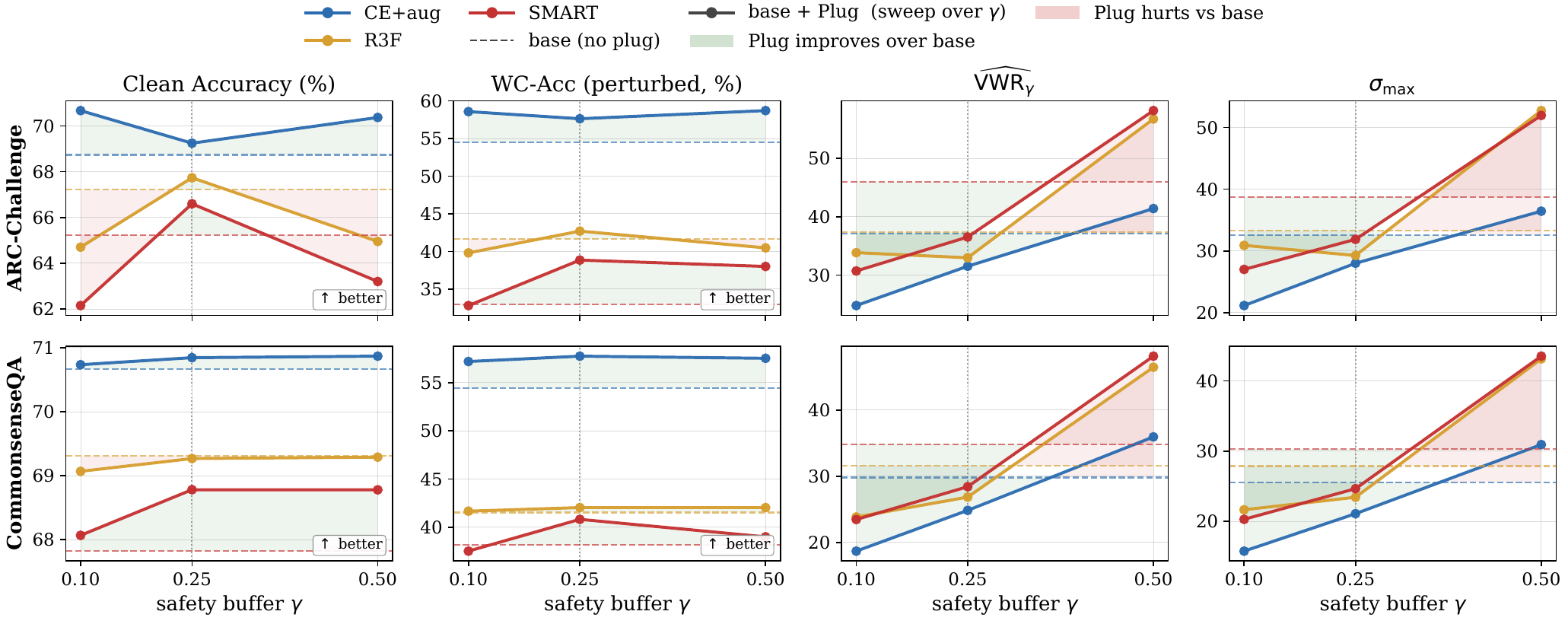}
% \caption{\textbf{Safety-buffer calibration for LLM robustness.}
% Results are averaged over three LLMs and two datasets. Dashed lines denote
% base learners, and solid curves denote plug-in counterparts as the
% validation-margin quantile $q$ varies. 
% Larger $q$ widens the
% vulnerable region, so risk values naturally increase in the two rightmost columns; comparisons are made
% within the same $q$. The intermediate buffer $\gamma_{25}$ provides the most
% stable trade-off and is used in the main experiments.}
\caption{\textbf{Safety-buffer calibration for LLM robustness.}
Results are averaged over three LLMs and two datasets. Dashed lines are base learners, and solid curves are plug-in counterparts. 
Larger $q$ widens the vulnerable region, so risk values increase by construction.
% comparisons are made within the same $q$. 
% We use the intermediate buffer $\gamma_{25}$ in the main experiments.
}
\label{fig:gamma_calibration}
\end{figure}

\subsection{Main results: reducing vulnerable error flow}
\label{sec:exp_main_results}
\label{sec:exp_llm_main}\label{sec:exp_vlm_main}

With the safety buffer fixed to $\gamma_{25}$, we evaluate whether the plug-in
realizes the control mechanism predicted by the theory and whether this control
translates into practical robustness gains. A limited sweep over $(\alpha,\beta)$
is reported in Appendix~\ref{app:alphabeta_sweep}; in the main experiments, we
use one shared default setting across models, datasets, and base learners.

Table~\ref{tab:llm_main_grid} reports the LLM results. Across the three models,
two datasets, and three robust adaptation objectives, FragileFlow reduces both
$\widehat{\mathrm{VWR}}_{\gamma}$ and $\widehat{\mathrm{VSR}}_{\gamma}$ in every
paired comparison. This directly verifies the intended mechanism: under the same
calibrated buffer, the plug-in compresses the dominant margin-vulnerable
error-flow structure rather than merely changing average accuracy. The downstream
metrics show the same trend in most cases. Perturbed worst-class accuracy
improves in most paired comparisons, while clean accuracy remains comparable to
the corresponding base learner. In the few cases where perturbed worst-class
accuracy does not increase, we keep the results unfiltered and still observe
lower values for both controlled risk quantities. This separation is expected:
the plug-in directly targets vulnerable probability flow, while its conversion
into deterministic worst-class accuracy also depends on the stability bridge in
Proposition~\ref{prop:deterministic_bridge} and on the optimization behavior of
the base learner.

\definecolor{plugblue}{RGB}{235,244,255}

\newcolumntype{P}{>{\columncolor{plugblue}\arraybackslash}c}

\setlength{\aboverulesep}{0pt}
\setlength{\belowrulesep}{1pt}
\setlength{\cmidrulesep}{0pt}

\begin{table*}[h]
\centering
\caption{\textbf{Main LLM results at the calibrated buffer $\gamma_{25}$.}
Results are reported over three paired seeds. Perturbed-side metrics average
over typo noise, distractor insertion, and format rewriting; clean accuracy is
measured on the original test set. Lower is better for
$\widehat{\mathrm{VWR}}_{\gamma}$ and $\widehat{\mathrm{VSR}}_{\gamma}$; higher
is better for accuracy. Bold indicates improvement over the paired base learner.}
\scriptsize
\setlength{\tabcolsep}{2.2pt}
\renewcommand{\arraystretch}{1.08}
\resizebox{0.9\textwidth}{!}{%
\begin{tabular}{ll l cP cP cP cP}
\toprule
\multirow{2}{*}{\textbf{Model}} &
\multirow{2}{*}{\textbf{Dataset}} &
\multirow{2}{*}{\textbf{Method}} &
\multicolumn{2}{c}{{$\widehat{\mathrm{VWR}}_\gamma$} $\downarrow$} &
\multicolumn{2}{c}{{$\widehat{\mathrm{VSR}}_{\gamma}$} $\downarrow$} &
\multicolumn{2}{c}{Ptb WC Acc $\uparrow$} &
\multicolumn{2}{c}{Clean Acc $\uparrow$} \\
\cmidrule(lr){4-5}\cmidrule(lr){6-7}\cmidrule(lr){8-9}\cmidrule(lr){10-11}
& & 
& base & \multicolumn{1}{P}{+plug}
& base & \multicolumn{1}{P}{+plug}
& base & \multicolumn{1}{P}{+plug}
& base & \multicolumn{1}{P}{+plug} \\
\midrule

\multirowcell{4}[0pt][c]{Qwen-0.5B} & \multirowcell{4}[0pt][c]{ARC-C}
& CE
& \emph{49.64 $\pm$ 1.73} & --
& \emph{45.06 $\pm$ 2.02} & --
& \emph{26.40 $\pm$ 2.17} & --
& \emph{52.33 $\pm$ 2.52} & -- \\
& & CE+aug
& 48.07 $\pm$ 1.55 & \textbf{38.46 $\pm$ 1.42}
& 43.34 $\pm$ 1.47 & \textbf{34.26 $\pm$ 2.15}
& 38.58 $\pm$ 8.92 & 38.31 $\pm$ 8.00
& 54.20 $\pm$ 3.41 & 53.87 $\pm$ 3.19 \\
& & R3F
& 49.68 $\pm$ 1.76 & \textbf{39.45 $\pm$ 1.07}
& 45.00 $\pm$ 2.07 & \textbf{35.64 $\pm$ 0.54}
& 26.40 $\pm$ 2.21 & \textbf{28.09 $\pm$ 3.34}
& 52.13 $\pm$ 2.44 & \textbf{53.93 $\pm$ 1.89} \\
& & SMART
& 53.70 $\pm$ 3.68 & \textbf{36.83 $\pm$ 4.00}
& 46.85 $\pm$ 0.84 & \textbf{34.11 $\pm$ 2.55}
& 22.04 $\pm$ 2.42 & \textbf{27.73 $\pm$ 2.82}
& 52.00 $\pm$ 0.53 & \textbf{53.13 $\pm$ 0.61} \\
\hdashline

\multirowcell{4}[0pt][c]{Qwen-0.5B} & \multirowcell{4}[0pt][c]{CSQA}
& CE
& \emph{40.85 $\pm$ 2.88} & --
& \emph{36.73 $\pm$ 0.70} & --
& \emph{28.56 $\pm$ 2.82} & --
& \emph{57.87 $\pm$ 2.32} & -- \\
& & CE+aug
& 38.32 $\pm$ 3.46 & \textbf{29.30 $\pm$ 2.73}
& 34.02 $\pm$ 2.17 & \textbf{25.79 $\pm$ 1.56}
& 40.33 $\pm$ 6.61 & \textbf{43.22 $\pm$ 7.15}
& 59.27 $\pm$ 1.92 & \textbf{59.73 $\pm$ 2.81} \\
& & R3F
& 40.84 $\pm$ 2.46 & \textbf{32.03 $\pm$ 2.78}
& 36.95 $\pm$ 1.13 & \textbf{28.59 $\pm$ 1.25}
& 28.67 $\pm$ 2.56 & 28.22 $\pm$ 2.72
& 57.53 $\pm$ 2.55 & \textbf{57.73 $\pm$ 2.02} \\
& & SMART
& 39.22 $\pm$ 2.83 & \textbf{31.47 $\pm$ 2.53}
& 34.80 $\pm$ 1.60 & \textbf{27.45 $\pm$ 1.50}
& 30.22 $\pm$ 2.37 & \textbf{32.89 $\pm$ 2.24}
& 59.33 $\pm$ 2.60 & \textbf{60.80 $\pm$ 2.78} \\
\hdashline

\multirowcell{4}[0pt][c]{Qwen-1.5B} & \multirowcell{4}[0pt][c]{ARC-C}
& CE
& \emph{30.87 $\pm$ 2.29} & --
& \emph{27.83 $\pm$ 2.57} & --
& \emph{48.98 $\pm$ 4.94} & --
& \emph{74.13 $\pm$ 2.39} & -- \\
& & CE+aug
& 30.34 $\pm$ 4.12 & \textbf{26.46 $\pm$ 1.52}
& 26.30 $\pm$ 1.58 & \textbf{23.77 $\pm$ 1.25}
& 63.84 $\pm$ 4.07 & \textbf{68.62 $\pm$ 3.69}
& 76.20 $\pm$ 1.71 & \textbf{76.93 $\pm$ 0.42} \\
& & R3F
& 30.95 $\pm$ 2.19 & \textbf{28.43 $\pm$ 2.63}
& 27.85 $\pm$ 2.57 & \textbf{26.33 $\pm$ 3.14}
& 47.89 $\pm$ 5.03 & \textbf{48.80 $\pm$ 4.97}
& 74.33 $\pm$ 2.25 & 74.20 $\pm$ 2.62 \\
& & SMART
& 31.61 $\pm$ 1.04 & \textbf{29.09 $\pm$ 1.89}
& 27.49 $\pm$ 0.45 & \textbf{25.40 $\pm$ 1.07}
& 50.18 $\pm$ 5.91 & \textbf{52.71 $\pm$ 5.80}
& 75.00 $\pm$ 0.53 & 74.93 $\pm$ 0.70 \\
\hdashline

\multirowcell{4}[0pt][c]{Qwen-1.5B} & \multirowcell{4}[0pt][c]{CSQA}
& CE
& \emph{28.12 $\pm$ 1.39} & --
& \emph{24.92 $\pm$ 1.93} & --
& \emph{45.67 $\pm$ 6.16} & --
& \emph{73.20 $\pm$ 0.92} & -- \\
& & CE+aug
& 25.20 $\pm$ 0.89 & \textbf{21.87 $\pm$ 1.25}
& 21.84 $\pm$ 0.51 & \textbf{18.86 $\pm$ 0.93}
& 62.22 $\pm$ 1.77 & \textbf{65.67 $\pm$ 1.80}
& 75.40 $\pm$ 0.72 & 75.13 $\pm$ 0.76 \\
& & R3F
& 27.82 $\pm$ 1.48 & \textbf{24.46 $\pm$ 1.60}
& 24.71 $\pm$ 1.78 & \textbf{21.64 $\pm$ 2.17}
& 44.67 $\pm$ 6.28 & \textbf{45.89 $\pm$ 6.44}
& 73.53 $\pm$ 0.81 & 73.40 $\pm$ 1.06 \\
& & SMART
& 27.57 $\pm$ 2.42 & \textbf{23.41 $\pm$ 1.78}
& 24.25 $\pm$ 1.46 & \textbf{20.26 $\pm$ 1.18}
& 49.78 $\pm$ 6.27 & \textbf{50.67 $\pm$ 5.52}
& 74.80 $\pm$ 2.51 & \textbf{74.87 $\pm$ 2.91} \\
\hdashline

\multirowcell{4}[0pt][c]{Mistral-7B} & \multirowcell{4}[0pt][c]{ARC-C}
& CE
& \emph{31.41 $\pm$ 4.52} & --
& \emph{27.13 $\pm$ 1.71} & --
& \emph{52.18 $\pm$ 4.97} & --
& \emph{75.33 $\pm$ 0.64} & -- \\
& & CE+aug
& 32.91 $\pm$ 4.15 & \textbf{29.60 $\pm$ 4.76}
& 28.09 $\pm$ 1.34 & \textbf{25.96 $\pm$ 2.01}
& 61.18 $\pm$ 4.21 & \textbf{65.96 $\pm$ 6.26}
& 75.80 $\pm$ 1.25 & \textbf{76.93 $\pm$ 3.56} \\
& & R3F
& 31.32 $\pm$ 3.77 & \textbf{30.99 $\pm$ 2.81}
& 26.99 $\pm$ 1.47 & \textbf{25.80 $\pm$ 1.12}
& 50.73 $\pm$ 4.33 & \textbf{51.20 $\pm$ 4.34}
& 75.20 $\pm$ 0.53 & 75.07 $\pm$ 0.46 \\
& & SMART
& 52.44 $\pm$ 5.96 & \textbf{43.60 $\pm$ 5.65}
& 41.88 $\pm$ 5.33 & \textbf{36.11 $\pm$ 5.93}
& 26.67 $\pm$ 9.98 & \textbf{36.09 $\pm$ 8.82}
& 68.67 $\pm$ 3.90 & \textbf{71.73 $\pm$ 4.32} \\
\hdashline

\multirowcell{4}[0pt][c]{Mistral-7B} & \multirowcell{4}[0pt][c]{CSQA}
& CE
& \emph{26.56 $\pm$ 3.27} & -- 
& \emph{22.13 $\pm$ 0.50} & --
& \emph{51.89 $\pm$ 8.69} & --
& \emph{76.40 $\pm$ 0.92} & -- \\
& & CE+aug
& 25.90 $\pm$ 2.06 & \textbf{23.34 $\pm$ 1.62}
& 20.81 $\pm$ 0.51 & \textbf{18.57 $\pm$ 0.93}
& 60.78 $\pm$ 4.81 & \textbf{64.33 $\pm$ 4.23}
& 77.33 $\pm$ 1.51 & \textbf{77.67 $\pm$ 1.60} \\
& & R3F
& 26.18 $\pm$ 3.92 & \textbf{24.04 $\pm$ 4.26}
& 21.94 $\pm$ 1.13 & \textbf{20.02 $\pm$ 1.62}
& 51.22 $\pm$ 8.55 & \textbf{52.00 $\pm$ 8.84}
& 76.87 $\pm$ 0.76 & 76.67 $\pm$ 0.81 \\
& & SMART
& 37.73 $\pm$ 3.25 & \textbf{30.37 $\pm$ 0.93}
& 31.78 $\pm$ 2.00 & \textbf{26.21 $\pm$ 0.87}
& 34.56 $\pm$ 6.84 & \textbf{38.89 $\pm$ 5.43}
& 69.33 $\pm$ 0.31 & \textbf{70.67 $\pm$ 0.31} \\
\bottomrule
\end{tabular}}

\label{tab:llm_main_grid}
\end{table*}

Table~\ref{tab:vlm_vit_main} provides a cross-modal test of the same mechanism.
Here the perturbation acts on the image input through PGD, rather than on the
text-side question or verbalizer interface. Despite this different perturbation
channel, FragileFlow again lowers $\widehat{\mathrm{VWR}}_{\gamma}$ and
$\widehat{\mathrm{VSR}}_{\gamma}$ across datasets and plug-in placements. Clean
accuracy is largely preserved, and the additional VLM metrics show that the gains
are not confined to the proposed risk scores: perturbed average accuracy and
worst-class readouts remain stable or improve in most settings.

% \setlength{\aboverulesep}{0pt}
% \setlength{\belowrulesep}{1pt}
% \setlength{\cmidrulesep}{0pt}
% \definecolor{pluggreen}{RGB}{238,248,238}
% \newcommand{\pluginrow}[7]{%

% & \cellcolor{pluggreen}#1

% & \cellcolor{pluggreen}#2

% & \cellcolor{pluggreen}#3

% & \cellcolor{pluggreen}#4

% & \cellcolor{pluggreen}#5

% & \cellcolor{pluggreen}#6

% & \cellcolor{pluggreen}#7 \\

% }
\definecolor{plugyellow}{RGB}{252,249,220}
\newcommand{\pluginrow}[7]{%
& \cellcolor{plugyellow}#1
& \cellcolor{plugyellow}#2
& \cellcolor{plugyellow}#3
& \cellcolor{plugyellow}#4
& \cellcolor{plugyellow}#5
& \cellcolor{plugyellow}#6
& \cellcolor{plugyellow}#7 \\
}

\begin{table}[h]
\centering
\small
\setlength{\tabcolsep}{5pt}
\renewcommand{\arraystretch}{1.18}
\caption{\textbf{Cross-modal VLM results on CLIP ViT-B/32.}
Results are reported for 16-shot LoRA adversarial adaptation over three seeds,
with robustness evaluated by 100-step PGD at $\varepsilon=1.0/255$. \textit{Inner},
\textit{outer}, and \textit{both} denote where the plug-in is applied during adaptation. Lower is
better for $\widehat{\mathrm{VWR}}_{\gamma}$ and
$\widehat{\mathrm{VSR}}_{\gamma}$; higher is better for accuracy metrics.}
\label{tab:vlm_vit_main}

\resizebox{0.82\linewidth}{!}{%
\begin{tabular}{lccccccc}
\toprule
\textbf{Dataset} & \textbf{Plugin type}
& $\widehat{\mathrm{VWR}}_\gamma$ $\downarrow$
& $\widehat{\mathrm{VSR}}_{\gamma}$ $\downarrow$
& Clean Acc $\uparrow$
& Ptb Acc $\uparrow$
& Clean WC $\uparrow$
& Ptb WC $\uparrow$ \\
\midrule

\multirow{4}{*}{DTD}
& --    
& $28.33 \pm 0.38$ 
& $38.97 \pm 4.80$ 
& $\mathbf{67.57 \pm 0.93}$ 
& $22.66 \pm 1.12$ 
& $28.70 \pm 3.46$ 
& $0.00$ \\
\pluginrow{\textit{Inner}}
{$27.22 \pm 0.47$}
{$37.47 \pm 5.04$}
{$67.53 \pm 1.04$}
{$\mathbf{22.73 \pm 1.10}$}
{$27.78 \pm 4.54$}
{$0.00$}
\pluginrow{\textit{Outer}}
{$26.91 \pm 0.59$}
{$36.25 \pm 4.83$}
{$67.55 \pm 0.68$}
{$22.56 \pm 1.19$}
{$\mathbf{29.63 \pm 4.72}$}
{$0.00$}
\pluginrow{\textit{Both}}
{$\mathbf{26.89 \pm 0.56}$}
{$\mathbf{36.15 \pm 4.34}$}
{$67.38 \pm 0.43$}
{$22.28 \pm 1.07$}
{$29.63 \pm 3.46$}
{$0.00$}

\hdashline

\multirow{4}{*}{OxfordPets}
& --    
& $63.53 \pm 2.09$ 
& $58.81 \pm 5.48$ 
& $89.52 \pm 0.08$ 
& $20.34 \pm 1.14$ 
& $54.00 \pm 2.45$ 
& $0.00$ \\
\pluginrow{\textit{Inner}}
{$63.27 \pm 2.15$}
{$58.02 \pm 5.94$}
{$89.38 \pm 0.19$}
{$20.72 \pm 1.15$}
{$54.00 \pm 2.94$}
{$\mathbf{0.33 \pm 0.47}$}
\pluginrow{\textit{Outer}}
{$61.60 \pm 2.81$}
{$57.12 \pm 5.65$}
{$89.79 \pm 0.34$}
{$20.70 \pm 1.04$}
{$\mathbf{56.67 \pm 3.40}$}
{$\mathbf{0.33 \pm 0.47}$}
\pluginrow{\textit{Both}}
{$\mathbf{61.53 \pm 2.79}$}
{$\mathbf{56.82 \pm 5.66}$}
{$\mathbf{89.82 \pm 0.44}$}
{$\mathbf{20.82 \pm 1.33}$}
{$56.33 \pm 1.25$}
{$\mathbf{0.33 \pm 0.47}$}

\hdashline

\multirow{4}{*}{Caltech101}
& --    
& $47.30 \pm 8.61$ 
& $22.20 \pm 3.74$ 
& $95.04 \pm 0.25$ 
& $64.60 \pm 2.72$ 
& $42.22 \pm 3.14$ 
& $\mathbf{5.19 \pm 4.09}$ \\
\pluginrow{\textit{Inner}}
{$45.40 \pm 7.20$}
{$\mathbf{20.79 \pm 3.61}$}
{$\mathbf{95.06 \pm 0.30}$}
{$64.72 \pm 2.60$}
{$40.00 \pm 3.44$}
{$\mathbf{5.19 \pm 4.09}$}
\pluginrow{\textit{Outer}}
{$46.14 \pm 9.77$}
{$20.90 \pm 4.31$}
{$94.97 \pm 0.32$}
{$64.72 \pm 2.96$}
{$\mathbf{42.78 \pm 2.83}$}
{$5.00 \pm 3.60$}
\pluginrow{\textit{Both}}
{$\mathbf{45.23 \pm 8.80}$}
{$20.89 \pm 4.14$}
{$94.90 \pm 0.26$}
{$\mathbf{64.75 \pm 2.99}$}
{$42.22 \pm 3.14$}
{$\mathbf{5.19 \pm 4.09}$}

\bottomrule
\end{tabular}}
\end{table}

The placement results further clarify how the regularizer acts. Applying the
plug-in to the inner PGD step shapes adversarial-example construction, whereas
applying it to the outer update directly regularizes the adapted model. These
placements lead to different downstream trade-offs, but the common pattern is
stable: under a fixed calibrated buffer, FragileFlow reduces concentrated
vulnerable error flow while preserving clean utility across both text-side and
image-side robustness settings.

\subsection{Mechanistic ablation: spectral compression and stability}
\label{sec:exp_ablation}

\begin{wrapfigure}[30]{r}{0.56\textwidth}
\vspace{-0.8em}
\centering

% \captionof{table}{\textbf{Compact stability-term ablation summary.}
% Rows report average shifts relative to the corresponding non-plug reference.}
\captionof{table}{\textbf{Compact stability-term ablation summary.}
Values are average shifts relative to the corresponding non-plug reference.
Risk changes are relative; accuracy changes are absolute percentage points.}

\label{tab:compact_beta_ablation}

\scriptsize
\setlength{\tabcolsep}{1.7pt}
\renewcommand{\arraystretch}{0.95}
\resizebox{\linewidth}{!}{%
\begin{tabular}{llccccc}
\toprule
Setting & Objective
& $\Delta\widehat{\mathrm{VSR}}_{\gamma}\downarrow$
& $\Delta\widehat{\mathrm{VWR}}_{\gamma}\downarrow$
& $\Delta\mathrm{WC}\uparrow$
& $\Delta\mathrm{PtbAcc}\uparrow$
& $\Delta\mathrm{CleanAcc}\uparrow$ \\
\midrule
\multirow{2}{*}{LLM}
& $R_{\mathrm{spec}}$ only $(\beta=0)$
& {$-13.02\%$}
& {$-13.14\%$}
& {$+0.48$ pp}
& {$+0.44$ pp}
& {$+0.42$ pp} \\
& $R_{\mathrm{spec}}+\beta R_{\mathrm{stab}}$
& {$-13.97\%$}
& {$-14.37\%$}
& {$+1.52$ pp}
& {$+0.68$ pp}
& {$+0.57$ pp} \\
% \hdashline
% \midrule
\hdashline
\multirow{2}{*}{VLM}
& $R_{\mathrm{spec}}$ only $(\beta=0)$
& {$-5.21\%$}
& {$-3.35\%$}
& {$+0.24$ pp}
& {$+0.33$ pp}
& {$-0.08$ pp} \\
& $R_{\mathrm{spec}}+\beta R_{\mathrm{stab}}$
& {$-6.47\%$}
& {$-4.35\%$}
& {$+0.28$ pp}
& {$+0.28$ pp}
& {$+0.02$ pp} \\
\bottomrule
\end{tabular}}

\vspace{0.7em}

\includegraphics[width=\linewidth]{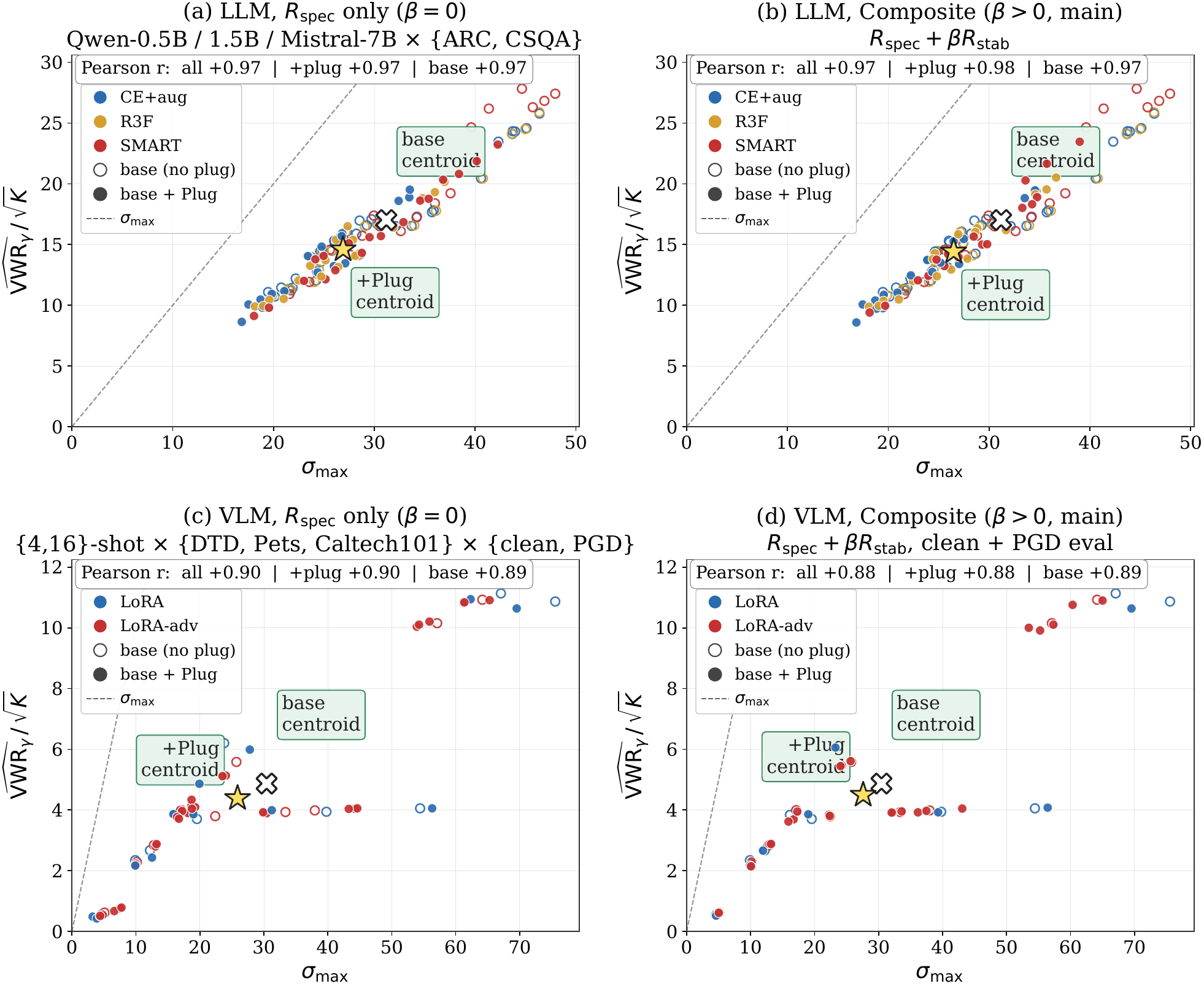}

\vspace{-0.8em}
% \caption{\textbf{Mechanistic ablation of spectral compression and stability.}
% The spectral-only objective compresses the dominant vulnerable-flow mode, while the stability term improves transfer to downstream robustness.}
\captionof{figure}{\textbf{Spectral compression with and without stability.}
Each panel plots $\sigma_{\max}=\widehat{\mathrm{VSR}}_{\gamma}$ against
$\widehat{\mathrm{VWR}}_{\gamma}/\sqrt{K}$. Centroids summarize the base and
plug-in runs under the spectral-only and composite objectives.}
\label{fig:b0_ablation}
\vspace{-1.0em}
\end{wrapfigure}

We finally isolate the two components of FragileFlow. The spectral term
$R_{\mathrm{spec}}$ directly penalizes the dominant mode of the
margin-aware error-flow matrix, measured by
$\sigma_{\max}=\widehat{\mathrm{VSR}}_{\gamma}$. The stability term
$R_{\mathrm{stab}}$ is not meant to replace this spectral mechanism; it is added
to make the controlled risk less sensitive to local coordinate perturbations,
which is the condition used by the deterministic bridge in
Proposition~\ref{prop:deterministic_bridge}.

Figure~\ref{fig:b0_ablation} visualizes this mechanism. Each panel plots
$\sigma_{\max}$ against the normalized vulnerable worst-class mass
$\widehat{\mathrm{VWR}}_{\gamma}/\sqrt{K}$. In both LLM and VLM settings, the
spectral-only objective ($\beta=0$) already moves the plug-in centroid toward the
lower-left region, showing that $R_{\mathrm{spec}}$ alone compresses the intended
vulnerable-flow geometry. The composite objective
$R_{\mathrm{spec}}+\beta R_{\mathrm{stab}}$ preserves the same direction rather
than changing the mechanism, which supports the intended complementary role of
the stability term.

Table~\ref{tab:compact_beta_ablation} shows the aggregate shifts relative to the
corresponding non-plug reference; risk metrics are reported as \textit{relative} changes,
and accuracy metrics as \textit{absolute} percentage-point changes. Detailed per-setting
LLM and VLM results are reported in Appendix
Tables~\ref{tab:llm_b0_ablation} and~\ref{tab:vlm_b0_ablation}. On LLMs,
$R_{\mathrm{spec}}$ alone substantially reduces the two risk measures, while the
larger downstream worst-class gain appears after adding $R_{\mathrm{stab}}$. On
VLMs, both variants reduce the risk measures under PGD evaluation, with the
composite objective giving the stronger risk reduction. These results show the intended functional separation: 
$R_{\mathrm{spec}}$ compresses the vulnerable error-flow structure, while 
$R_{\mathrm{stab}}$ improves the reliability of its transfer to evaluation accuracy.

% Overall, these experiments validate the intended design of our method. The
% calibrated buffer gives a stable risk definition, the plug-in consistently
% reduces margin-vulnerable error flow, and the ablations show that the spectral
% and stability terms play complementary roles. Together, these results support
% our central claim: the proposed design provides an effective and practical way
% to improve worst-class robustness while preserving clean utility.
\textbf{Experiment summary.}  Taken together, the experiments support the proposed risk-control chain: the
calibrated buffer defines a stable vulnerable region, the plug-in reduces the
targeted error-flow risks, and the ablations confirm the complementary roles of
$R_{\mathrm{spec}}$ and $R_{\mathrm{stab}}$. This improves worst-class robustness
while largely preserving clean utility.

\section{Related Work}
\label{sec:related}

\subsection{Robustness to Input Perturbations}

Robustness to input perturbations is commonly studied through adversarial training, consistency regularization, and flatness-based objectives. 
Adversarial training improves worst-case robustness by training on perturbed examples
\citep{DBLP:journals/corr/GoodfellowSS14,madry2018towards,zhang2019theoretically}, while SMART, R3F, adversarial weight perturbation, and sharpness-aware training stabilize predictions under input, embedding, or parameter perturbations
\citep{jiang2020smart,aghajanyan2021better,NEURIPS2020_1ef91c21,foret2021sharpness}. 
These objectives are effective, but they usually aggregate robustness across samples or losses.
Recent work extends these ideas to LLMs and VLMs. 
For LLMs, prompt perturbations, distractors, typos, and instruction edits motivate perturbation-aware fine-tuning and prompt-consistency learning
\citep{qiang2024prompt,gupta2024evaluating,agrawal2025enhancing}. 
For VLMs, robust adaptation often perturbs visual inputs while tuning lightweight components such as adapters, LoRA modules, or vision encoders
\citep{pmlr-v235-schlarmann24a,ghiasvand2025fewshotadversariallowrankfinetuning,oh2024towards}. 
A related line studies class-wise and worst-class robustness. 
Prior work shows that adversarial training can create large robustness disparities across classes
\citep{xu2021robust,tian2021analysis}. 
BAT, CFA, and WAT reduce such disparities through class balancing, class-specific adversarial configurations, or direct worst-class optimization
\citep{sun2023improving,wei2023classwise,10.1609/aaai.v37i12.26749}. 
Confusional spectral regularization further regularizes hard adversarial confusion for worst-class robust fairness
\citep{jin2025enhancing}. 
FragileFlow shares the weakest-class motivation, but controls a margin-aware probability-flow matrix in LLM/VLM adaptation, capturing where probability mass moves before a prediction fails.

\subsection{PAC-Bayes for Neural Networks}

% PAC-Bayes theory bounds the population risk of randomized predictors by combining empirical risk
% with a posterior--prior complexity term \citep{mcallester1999pac,catoni2007pac}. For neural
% networks, it has been used to analyze generalization through flatness, compression, spectral
% complexity, and parameter perturbations
% \citep{dziugaite2017computing,neyshabur2018pac,arora2018stronger}.

% Recent PAC-Bayes analyses further study modern large models and robust generalization.
% Compression-based bounds suggest that pretrained or adapted LLMs can have much smaller effective
% complexity than their raw parameter count
% \citep{lotfi2024nonvacuous,li2025sr}. Robust PAC-Bayes objectives connect bound minimization to
% regularization of the robust loss geometry \citep{wang2023improving}. Closest in spirit,
% principal-eigenvalue regularization relates the spectral structure of a confusion matrix to
% worst-class certified robustness in smoothed classifiers \citep{jin2025principal}.

% Our analysis follows this spectral viewpoint but targets verbalizer-based LLM/VLM adaptation under
% practical perturbations. We do not optimize the full PAC-Bayes bound directly; instead, the theory
% identifies a controllable empirical object: the spectral norm of a margin-aware error-flow matrix.
% Under logit stability, controlling this object gives a conservative route from empirical spectral
% control to deterministic worst-class robustness.

PAC-Bayes theory is a classical framework for connecting empirical performance with
population risk through a posterior--prior complexity term
\citep{mcallester1999pac,catoni2007pac}. 
For neural networks, it has been used to analyze generalization through flatness,
compression, margin-based complexity, and parameter perturbations
\citep{dziugaite2017computing,neyshabur2018pac,arora2018stronger}.
Recent work has gradually extended this framework to larger models and robust learning.
Compression-based analyses suggest that pretrained or adapted LLMs can have much smaller
effective complexity than their raw parameter count
\citep{lotfi2024nonvacuous,li2026promptperturbationsbreakgeneration}. 
PAC-driven fine-tuning further uses posterior perturbations to guide pretrained language
model adaptation \citep{liu2023pactuning}. 
Robust PAC-Bayes analyses connect bound minimization to adversarial or certified robustness
\citep{wang2023improving,jin2025enhancing}.
Our work is inspired by this line of analysis, but uses PAC-Bayes for a different purpose.
Rather than bounding standard generalization error, compression behavior, fine-tuning
performance, or average robust risk, we analyze the margin-aware error-flow matrix induced
by finite-option LLM/VLM adaptation under perturbation. 
This lets us move beyond a purely empirical regularizer: the same object optimized by
FragileFlow also appears in a PAC-Bayes control route for vulnerable worst-class risk.
% Under logit stability, this route further connects the posterior risk to deterministic
% worst-class robustness at test time.
\section{Conclusion}
\label{sec:conclusion}

In this paper, we identified a problem that is often missed by average robustness evaluation: under perturbation, some classes can become fragile by consistently leaking probability mass toward recurring wrong competitors. We introduced \emph{FragileFlow} to capture and reduce this vulnerable error flow, and provided a PAC-Bayes control route to connet the empirical regularized object to worst-class behavior at the population level. Experiments on LLM and VLM adaptation validate this view.
More broadly, our results suggest that robustness should be studied not only through final prediction errors, but also through the structure that emerges before those errors occur. The probability-flow view gives a natural way to inspect how fragility concentrates across classes and points toward more interpretable robustness analysis. An immediate next step, which we are currently pursuing, is to understand how this form of regularization acts inside the model, including whether it suppresses fragile representation directions, changes internal error-flow pathways, or extends from finite-option prediction to generated tokens and semantic decision states.

\begin{ack}
This work is partially funded by the European Union (under grant agreement ID 101212818). Views and opinions expressed are however those of the author(s) only and do not necessarily reflect those of the European Union or European Health and Digital Executive Agency (HADEA). Neither the European Union nor the granting authority can be held responsible for them. This work is partially supported by Innovate UK through AI-PASSPORT under Grant 10126404. This work was awarded a grant by the AI Security Institute (AISI) via the Alignment Project (Rare-Event Estimation in Large Language Models via Subset Simulation) and funded by EPSRC. Yi's contribution is partially supported through the Royal Society international exchanges programme and in part by the Engineering and Physical Sciences Research Council, through funding from RAi UK [EP/Y009800/1].
\end{ack}

% \clearpage
\bibliographystyle{plain}
\bibliography{refs}
%%%%%%%%%%%%%%%%%%%%%%%%%%%%%%%%%%%%%%%%%%%%%%%%%%%%%%%%%%%%
\clearpage
\appendix
% Appendix (proofs and additional details).
% NOTE: This file is included after \appendix in main.tex.

\section{Proofs and Additional Details for Section~\ref{sec:methodology}}
\label{app:sec_prelim_proofs}

\subsection{Proof of margin-stability}
\label{app:proof_margin_buffer}
\begin{proof}
Let $y\in[K]$ be the true class and define $c(x):=\max_{k\neq y} s_\theta(x,k)$ and $c(x'):=\max_{k\neq y} s_\theta(x',k)$.
By definition, $\Delta_\theta(x,y)=s_\theta(x,y)-c(x)$ and $\Delta_\theta(x',y)=s_\theta(x',y)-c(x')$.
Assume
\[
\max_{k\in[K]}|s_\theta(x,k)-s_\theta(x',k)|\le \varepsilon.
\]
Then $|s_\theta(x,y)-s_\theta(x',y)|\le \varepsilon$.
Moreover, for every $k\neq y$, we have $s_\theta(x',k)\le s_\theta(x,k)+\varepsilon$, so taking the maximum over $k\neq y$ yields $c(x')\le c(x)+\varepsilon$.
Hence
\[
\Delta_\theta(x',y)
=
s_\theta(x',y)-c(x')
\ge
(s_\theta(x,y)-\varepsilon) - (c(x)+\varepsilon)
=
\Delta_\theta(x,y)-2\varepsilon,
\]
.
% If $\Delta_\theta(x,y)\ge \gamma$ and $\varepsilon\le \gamma/2$, then $\Delta_\theta(x',y)\ge 0$, implying $\hat y^{\mathrm{det}}_\theta(x')=y$.
If $\Delta_\theta(x,y)> \gamma$ and $\varepsilon< \gamma/2$, then
$\Delta_\theta(x',y)>0$, and hence the true class is the unique maximizer
among the verbalizer classes. Therefore
$\hat y^{\mathrm{det}}_\theta(x')=y$.
% The non-strict version, $\Delta_\theta(x,y)\ge \gamma$ and
% $\varepsilon\le \gamma/2$, only guarantees $\Delta_\theta(x',y)\ge 0$.
% In that boundary case, prediction preservation additionally depends on the
% tie-breaking rule.
\end{proof}

This also lays the groundwork for the subsequent deterministic bridge. A later proposition demonstrates that when the logits of the posterior samples are sufficiently close to those of the mean model, the posterior vulnerable risk can be linked to the deterministic worst-class risk. This section clarifies in advance, under a fixed input perturbation, that the margin buffer can absorb a limited amount of logit deviation.

\subsection{Proof of spectral norm control}
\label{app:proof_l1_spec}
\begin{proof}
Recall $\|A\|_1=\max_{j}\sum_i |A_{ij}|$ and $\|A\|_2=\max_{\|u\|_2=1}\|Au\|_2$.
Let $e_j$ be the $j$-th standard basis vector.
Then
\[
\|A\|_1
=
\max_j \|Ae_j\|_1
\le
\max_j \sqrt{K}\,\|Ae_j\|_2
\le
\sqrt{K}\,\|A\|_2,
\]
where we used $\|z\|_1\le \sqrt{K}\|z\|_2$ for $z\in\mathbb{R}^K$.
\end{proof}

\label{app:sec_pacbayes_proofs}

\subsection{Proof of Theorem~\ref{thm:pacbayes_spectral}}
\label{app:proof_pacbayes_spectral}
\begin{proof}
We condition throughout on the realized class counts $\{m_j\}_{j=1}^{K}$ and on the perturbation
protocol that induces $D'$ and $S'$. In particular, for this PAC-Bayes statement, the perturbed sample is
treated as fixed with respect to the posterior $Q$ being evaluated. This ensures that the quantities below
are ordinary bounded statistics indexed by the trainable-coordinate random variable $w$.

\medskip
\noindent\textbf{Step 1: Spectral norm as a bilinear form.}
For any matrix $A\in\mathbb{R}^{K\times K}$,
\[
\|A\|_2
=
\sup_{\|u\|_2=1,\|v\|_2=1}
\left|u^\top A v\right|.
\]
We apply this to
\[
A_Q
:=
\bar M^{Q}_{D',\gamma}
-
\bar M^{Q}_{S',\gamma}.
\]

\medskip
\noindent\textbf{Step 2: Stratified bilinear statistic for fixed $u,v$.}
Fix unit vectors $u,v\in\mathbb{R}^{K}$.
For each class $j\in[K]$, define
\[
\ell_{u,j}(w;x')
:=
\sum_{i\neq j}
u_i\,
g_{\gamma,\kappa}^{\theta(w)}(x',j)\,
p_{\theta(w)}(i\mid x').
\]
Let $D'_j$ denote the class-$j$ conditional distribution, and write the class-$j$ sample in $S'$ as
$\{x'_{j,r}\}_{r=1}^{m_j}$.
For fixed $w$, define the stratified bilinear generalization gap
\[
Z_{u,v}(w;S')
:=
\sum_{j=1}^{K}
v_j
\left(
\mathbb{E}_{x'\sim D'_j}\ell_{u,j}(w;x')
-
\frac{1}{m_j}\sum_{r=1}^{m_j}\ell_{u,j}(w;x'_{j,r})
\right).
\]
By expanding the bilinear form column by column, we have
\[
u^\top
\big(
\bar M^{Q}_{D',\gamma}
-
\bar M^{Q}_{S',\gamma}
\big)
v
=
\mathbb{E}_{\widetilde w\sim Q}
\left[
Z_{u,v}(\widetilde w;S')
\right].
\]

\medskip
\noindent\textbf{Step 3: Boundedness.}
For fixed $j$, define
\[
a_i
:=
g_{\gamma,\kappa}^{\theta(w)}(x',j)\,
p_{\theta(w)}(i\mid x'),
\qquad i\neq j.
\]
Then $a_i\ge 0$ and
\[
\sum_{i\neq j}a_i
\le
g_{\gamma,\kappa}^{\theta(w)}(x',j)
\sum_{i\neq j}p_{\theta(w)}(i\mid x')
\le 1.
\]
Since $\|u\|_2=1$ implies $\|u\|_\infty\le 1$, we obtain
\[
\left|
\ell_{u,j}(w;x')
\right|
=
\left|
\sum_{i\neq j}u_i a_i
\right|
\le
\|u\|_\infty\sum_{i\neq j}a_i
\le 1.
\]
Thus each class-conditional term is bounded in $[-1,1]$.

\medskip
\noindent\textbf{Step 4: PAC-Bayes bound for the stratified bilinear gap.}
For fixed $u,v,w$, the random variables in $Z_{u,v}(w;S')$ are independent across the
stratified class samples, conditional on the class counts.
Since each $\ell_{u,j}(w;x')\in[-1,1]$, Hoeffding's lemma gives, for any $\lambda>0$,
\[
\mathbb{E}_{S'}
\exp\left(
\lambda Z_{u,v}(w;S')
\right)
\le
\exp\left(
\frac{\lambda^2}{2}
\sum_{j=1}^{K}\frac{v_j^2}{m_j}
\right),
\]
and the same bound holds for $-Z_{u,v}(w;S')$.
By the standard PAC-Bayes change-of-measure argument, with probability at least
$1-\delta_0$ over $S'$, simultaneously for all posteriors $Q$,
\[
\left|
\mathbb{E}_{\widetilde w\sim Q}
Z_{u,v}(\widetilde w;S')
\right|
\le
\sqrt{
2\left(
\sum_{j=1}^{K}\frac{v_j^2}{m_j}
\right)
\left(
\mathrm{KL}(Q\|P)+\ln\frac{2}{\delta_0}
\right)
}.
\]
Since $m_j\ge m_{\min}$ and $\|v\|_2=1$,
\[
\sum_{j=1}^{K}\frac{v_j^2}{m_j}
\le
\frac{1}{m_{\min}}.
\]
Therefore, for fixed $u,v$, with probability at least $1-\delta_0$, simultaneously for all $Q$,
\[
\left|
u^\top
\big(
\bar M^{Q}_{D',\gamma}
-
\bar M^{Q}_{S',\gamma}
\big)
v
\right|
\le
\sqrt{
\frac{
2\left(
\mathrm{KL}(Q\|P)+\ln\frac{2}{\delta_0}
\right)
}{
m_{\min}
}
}.
\]

\medskip
\noindent\textbf{Step 5: Epsilon-net reduction.}
Let $\mathcal{N}_{1/4}$ be a $1/4$-net of the Euclidean unit sphere in $\mathbb{R}^{K}$.
We use the standard bound
\[
|\mathcal{N}_{1/4}|\le 9^{K},
\]
and the standard net-to-spectrum reduction
\[
\|A\|_2
\le
2
\max_{\hat u,\hat v\in\mathcal{N}_{1/4}}
\left|
\hat u^\top A\hat v
\right|.
\]
Apply the fixed-$(u,v)$ bound to all pairs
$(\hat u,\hat v)\in\mathcal{N}_{1/4}\times\mathcal{N}_{1/4}$ with
\[
\delta_0
:=
\frac{\delta}{|\mathcal{N}_{1/4}|^2}.
\]
By a union bound, with probability at least $1-\delta$, simultaneously for all posteriors $Q$ and all net points,
\[
\left|
\hat u^\top
\big(
\bar M^{Q}_{D',\gamma}
-
\bar M^{Q}_{S',\gamma}
\big)
\hat v
\right|
\le
\sqrt{
\frac{
2\left(
\mathrm{KL}(Q\|P)
+
\ln\frac{2|\mathcal{N}_{1/4}|^2}{\delta}
\right)
}{
m_{\min}
}
}.
\]
Using $|\mathcal{N}_{1/4}|\le 9^K$, we get
\[
\ln\frac{2|\mathcal{N}_{1/4}|^2}{\delta}
\le
2K\ln 9+\ln\frac{2}{\delta}.
\]
Thus,
\[
\left\|
\bar M^{Q}_{D',\gamma}
-
\bar M^{Q}_{S',\gamma}
\right\|_2
\le
2
\sqrt{
\frac{
2\left(
\mathrm{KL}(Q\|P)+2K\ln 9+\ln\frac{2}{\delta}
\right)
}{
m_{\min}
}
}.
\]

\medskip
\noindent\textbf{Step 6: From spectral deviation to vulnerable worst-class risk.}
By the norm conversion in Appendix~\ref{app:proof_l1_spec},
\[
\mathrm{VWR}_\gamma(Q;D')
=
\left\|
\bar M^{Q}_{D',\gamma}
\right\|_1
\le
\sqrt K
\left\|
\bar M^{Q}_{D',\gamma}
\right\|_2.
\]
Using the triangle inequality,
\[
\left\|
\bar M^{Q}_{D',\gamma}
\right\|_2
\le
\left\|
\bar M^{Q}_{S',\gamma}
\right\|_2
+
\left\|
\bar M^{Q}_{D',\gamma}
-
\bar M^{Q}_{S',\gamma}
\right\|_2.
\]
Therefore,
\[
\mathrm{VWR}_\gamma(Q;D')
\le
\sqrt K\,\mathrm{VSR}_\gamma(Q;S')
+
2\sqrt{
\frac{
2K\left(
\mathrm{KL}(Q\|P)+2K\ln 9+\ln\frac{2}{\delta}
\right)
}{
m_{\min}
}
}.
\]
This completes the proof.
\end{proof}

\subsection{Case-wise view of the stability event}
\label{app:stability_event_cases}

Let
\[
\Delta_\mu:=\Delta_{\theta(\mu)}(x',y),
\qquad
\widetilde{\Delta}:=\Delta_{\theta(\tilde w)}(x',y),
\]
and define the stability event
\[
\mathcal{E}_{\mathrm{stab}}
:=
\left\{
\Xi_Q(\mu,\tilde w)\le \gamma/2
\right\}.
\]
On $\mathcal{E}_{\mathrm{stab}}$, every verbalizer score changes by at most
$\gamma/2$. Since the margin is the difference between the true-class score
and the largest wrong-class score, the margin can change by at most $\gamma$:
\[
|\widetilde{\Delta}-\Delta_\mu|\le \gamma .
\]
Equivalently,
\[
\Delta_\mu-\gamma
\le
\widetilde{\Delta}
\le
\Delta_\mu+\gamma .
\]
We now spell out the implications case by case.

\paragraph{Case 1: $\Delta_\mu\le 0$.}
The deterministic mean model makes an error. On
$\mathcal{E}_{\mathrm{stab}}$, the upper bound gives
\[
\widetilde{\Delta}
\le
\Delta_\mu+\gamma
\le
\gamma .
\]
Thus the posterior sample lies in the $\gamma$-vulnerable region
$\{\widetilde{\Delta}\le\gamma\}$. This is the only direction needed for
Proposition~\ref{prop:deterministic_bridge}: a deterministic error of the
mean model is covered by the posterior vulnerable event, except when
$\mathcal{E}_{\mathrm{stab}}$ fails.

\paragraph{Case 2: $0<\Delta_\mu\le\gamma$.}
The mean model is correct but already near the decision boundary. On
$\mathcal{E}_{\mathrm{stab}}$,
\[
\Delta_\mu-\gamma
\le
\widetilde{\Delta}
\le
\Delta_\mu+\gamma .
\]
Since $0<\Delta_\mu\le\gamma$, this implies
\[
-\gamma
<
\Delta_\mu-\gamma
\le
\widetilde{\Delta}
\le
\Delta_\mu+\gamma
\le
2\gamma .
\]
Hence the posterior sample may either remain correct, flip, or enter the
$\gamma$-vulnerable region depending on the realized posterior shift. Counting
such mean-correct but noise-sensitive examples in the posterior vulnerable
risk is \textbf{conservative} and does not invalidate the upper bound on deterministic
worst-class risk.

\paragraph{Case 3: $\Delta_\mu>\gamma$.}
The mean model has a positive safety margin. On
$\mathcal{E}_{\mathrm{stab}}$, the lower bound gives
\[
\widetilde{\Delta}
\ge
\Delta_\mu-\gamma
>
0 ,
\]
so posterior sampling cannot flip the prediction under the stable event.
There are two subcases. 

If $\Delta_\mu>2\gamma$, then
\[
\widetilde{\Delta}
\ge
\Delta_\mu-\gamma
>
\gamma ,
\]
so the stable posterior sample stays outside the vulnerable region.

If $\gamma<\Delta_\mu\le 2\gamma$, then
\[
0
<
\Delta_\mu-\gamma
\le
\gamma ,
\]
so a stable posterior sample may still enter the vulnerable region
$\{\widetilde{\Delta}\le\gamma\}$, reflecting local sensitivity around the
mean model. 
% Therefore, this is the only situation that
% $\mathcal{E}_{\mathrm{stab}}$ fails. 
Therefore, this conservative counting is harmless for the upper bound, while genuine large posterior-induced margin shifts are accounted for by the failure probability $\rho$.

This unstable posterior-induced margin
shifts are precisely the events accounted for by the failure probability
$\rho$, and they motivate the stability regularizer used in the main method.

\subsection{Risk-level Gibbs-to-deterministic bridge}
\label{app:bridge_gibbs_det}
\label{app:deterministic_bridge}

For the bridge proof, define the posterior margin-failure risk
\begin{equation}
R_{\mathrm{mf},\gamma}(Q;D')
:=
\max_{j\in[K]}
\mathbb{E}_{x'\sim D'_j}
\left[
\Pr_{\widetilde w\sim Q}
\left(
\Delta_{\theta(\widetilde w)}(x',j)\le \gamma
\right)
\right],
\end{equation}
where $D'_j$ denotes the class-$j$ conditional perturbed distribution.

\begin{lemma}[Risk-level bridge under logit stability]
\label{lem:gibbs_det_bridge}
\label{lem:risk_level_bridge}
Fix $\gamma\ge 0$ and let $Q$ be any posterior over trainable coordinates with finite mean $\mu=\mathbb{E}_{\tilde w\sim Q}[\tilde w]$.
Assume
\[
\Pr_{\tilde w\sim Q}
\big(\Xi_Q(\mu,\tilde w)\le \gamma/2\big)
\ge 1-\rho.
\]
Then
\[
\mathrm{WCR}^{\mathrm{det}}(\theta(\mu);D')
\le
R_{\mathrm{mf},\gamma}(Q;D')
\;+\;
\rho.
\]
\end{lemma}
\begin{proof}
Fix $(x',y)$ and define
\[
E
:=
\{\hat y^{\mathrm{det}}_{\theta(\mu)}(x')\neq y\},
\qquad
F(\tilde w)
:=
\{\Delta_{\theta(\tilde w)}(x',y)\le \gamma\},
\]
and the stability event
\[
S(\tilde w)
:=
\{\Xi_Q(\mu,\tilde w)\le \gamma/2\}.
\]
If $E$ occurs, then $\Delta_{\theta(\mu)}(x',y)\le 0$.
On $S(\tilde w)$, the score perturbation between $\theta(\mu)$ and $\theta(\tilde w)$ is bounded by $\gamma/2$ for every class.
By the margin-stability argument in Appendix~\ref{app:proof_margin_buffer},
with the two score vectors swapped, we obtain
\[
\Delta_{\theta(\tilde w)}(x',y)
\le
\Delta_{\theta(\mu)}(x',y)+\gamma
\le \gamma.
\]
Hence $E\cap S(\tilde w)\subseteq F(\tilde w)$, so pointwise
\[
\mathbf{1}(E)
\le
\mathbf{1}(F(\tilde w))
+
\mathbf{1}(S(\tilde w)^c).
\]
Taking expectation over $\tilde w\sim Q$ gives
\[
\mathbf{1}(E)
\le
\mathbb{E}_{\tilde w\sim Q}
\big[
\mathbf{1}\{\Delta_{\theta(\tilde w)}(x',y)\le \gamma\}
\big]
+
\rho.
\]
Now take conditional expectation over $x'\sim D'_j$ for each class $j$, and then maximize over $j\in[K]$.
This yields the stated risk-level bridge.
No entrywise matrix domination is claimed.
\end{proof}

\subsection{Margin failure is controlled by vulnerable mass}
\label{app:proof_margin_failure_to_vwr}

We now prove the second part of the deterministic-risk bridge used in the
main text.

\begin{lemma}[Margin failure controlled by margin-aware vulnerable mass]
\label{lem:margin_failure_to_vwr}
Fix $\gamma\ge 0$. Suppose the gate satisfies
\[
g_{\gamma,\kappa}^{\theta}(x',j)\ge \eta
\quad
\text{whenever}
\quad
\Delta_{\theta}(x',j)\le \gamma
\]
for some $\eta>0$. Then
\[
R_{\mathrm{mf},\gamma}(Q;D')
\le
\eta^{-1}(1+e^\gamma)\,
\mathrm{VWR}_\gamma(Q;D').
\]
For the sigmoid gate
$g_{\gamma,\kappa}^{\theta}(x',j)
=\sigma((\gamma-\Delta_\theta(x',j))/\kappa)$, one may take
$\eta=1/2$.
\end{lemma}

\begin{proof}
Fix a posterior sample $\tilde w\sim Q$, a perturbed example $x'$, and a
class $j$. For brevity write $\theta=\theta(\tilde w)$.
If $\Delta_\theta(x',j)\le \gamma$, then by the definition of the margin
there exists some competitor $i\neq j$ such that
\[
s_\theta(x',i)
\ge
s_\theta(x',j)-\gamma .
\]
Therefore,
\[
\exp(s_\theta(x',i))
\ge
e^{-\gamma}\exp(s_\theta(x',j)).
\]
Using only this competitor in the softmax denominator gives
\[
p_\theta(j\mid x')
=
\frac{\exp(s_\theta(x',j))}
{\sum_{k=1}^K \exp(s_\theta(x',k))}
\le
\frac{\exp(s_\theta(x',j))}
{\exp(s_\theta(x',j))+\exp(s_\theta(x',i))}
\le
\frac{1}{1+e^{-\gamma}}.
\]
Hence
\[
1-p_\theta(j\mid x')
\ge
\frac{1}{1+e^\gamma}.
\]
Moreover, on the same event $\Delta_\theta(x',j)\le\gamma$, the gate
satisfies $g_{\gamma,\kappa}^{\theta}(x',j)\ge\eta$. Thus,
\[
g_{\gamma,\kappa}^{\theta}(x',j)
\bigl(1-p_\theta(j\mid x')\bigr)
\ge
\frac{\eta}{1+e^\gamma}.
\]
Equivalently, pointwise,
\[
\mathbf{1}\{\Delta_\theta(x',j)\le\gamma\}
\le
\eta^{-1}(1+e^\gamma)
g_{\gamma,\kappa}^{\theta}(x',j)
\bigl(1-p_\theta(j\mid x')\bigr).
\]
Taking expectation over $\tilde w\sim Q$ and over the class-conditional
distribution $x'\sim D'_j$, and then maximizing over $j\in[K]$, gives
\[
R_{\mathrm{mf},\gamma}(Q;D')
\le
\eta^{-1}(1+e^\gamma)
\mathrm{VWR}_\gamma(Q;D').
\]
\end{proof}

\paragraph{Combined consequence.}
Combining Lemma~\ref{lem:margin_failure_to_vwr} with the risk-level bridge
in Lemma~\ref{lem:risk_level_bridge} gives
\[
\mathrm{WCR}^{\mathrm{det}}(\theta(\mu);D')
\le
\eta^{-1}(1+e^\gamma)
\mathrm{VWR}_{\gamma}(Q;D')
+
\rho .
\]
% Combining this inequality with Theorem~\ref{thm:pacbayes_spectral} further
% yields, with the same high-probability event as the theorem,
% \[
% \mathrm{WCR}^{\mathrm{det}}(\theta(\mu);D')
% \le
% \eta^{-1}(1+e^\gamma)
% \left[
% \sqrt K
% \left\|
% \bar M_{S',\gamma}^{Q}
% \right\|_2
% +
% 2K
% \sqrt{
% \frac{
% \mathrm{KL}(Q\|P)+2K\ln 9+\ln\frac{2K}{\delta}
% }{
% 2m_{\min}
% }
% }
% \right]
% +\rho .
% \]
Combining this inequality with Theorem~\ref{thm:pacbayes_spectral} further yields, with the same high-probability event as the theorem,
\[
\mathrm{WCR}^{\mathrm{det}}(\theta(\mu);D')
\le
\eta^{-1}(1+e^\gamma)
\left[
\sqrt{K}\,
\mathrm{VSR}_\gamma(Q;S')
+
2
\sqrt{
\frac{
2K\left(
\mathrm{KL}(Q\|P)+2K\ln 9+\ln\frac{2}{\delta}
\right)
}{
m_{\min}
}
}
\right]
+\rho .
\]
This is the conservative PAC-Bayes route from the empirical margin-aware
spectral structure to deterministic worst-class risk.

% \section{Additional Details for Section~\ref{sec:methodology}}
% \label{app:sec_method_details}

\subsection{Training algorithm (classification plug-in; experimental LoRA instantiation)}
\label{app:sodu_algorithm}
\begin{algorithm}[t]
\caption{Plug-in spectral safety control for verbalizer classification}
\label{alg:plugin_training}
\begin{algorithmic}[1]
\REQUIRE Dataset $S$, perturbation mechanism $\mathcal{U}(\cdot)$, safety buffer $\gamma$, gate temperature $\kappa$, base objective $\mathcal{L}_{\mathrm{base}}$, spectral weight $\alpha$, stability weight $\beta$, power-iteration steps $T_{\mathrm{pi}}$, refresh interval $N$, numerical constant $\varepsilon_{\mathrm{spec}}$, and local-coordinate perturbation scale $\sigma_Q$.
\STATE Initialize trainable coordinates $\phi$ and a unit vector $v\in\mathbb{R}^{K}$.
\FOR{each training step $t$}
    \STATE Sample a mini-batch $B=\{(x_q,y_q)\}$ from $S$.
    \STATE Construct perturbed inputs $x'_q\in\mathcal{U}(x_q)$ and form $B'=\{(x'_q,y_q)\}$.
    \STATE Compute verbalizer probabilities $p_{\theta(\phi)}(\cdot\mid x'_q)$, margins $\Delta_{\theta(\phi)}(x'_q,y_q)$, and gates $g^{\theta(\phi)}_{\gamma,\kappa}(x'_q,y_q)$.
    \STATE Build the differentiable batch matrix $\widetilde M_{\phi}^{\mathrm{batch},\gamma}$ via Eq.~\eqref{eq:soft_M_batch}.
    Columns whose classes are absent from $B$ are masked out for this update; the matrix dimension remains $K\times K$.
    \IF{$t \bmod N = 0$}
        \STATE Update $v$ by $T_{\mathrm{pi}}$ steps of power iteration on
        $(\widetilde M_{\phi}^{\mathrm{batch},\gamma})^\top
        \widetilde M_{\phi}^{\mathrm{batch},\gamma}$, and normalize $v$.
    \ENDIF
    \STATE Treat $v$ as fixed within the current refresh window and estimate
    \[
    \mathcal{R}_{\mathrm{spec}}
    =
    \sqrt{
    v^\top
    (\widetilde M_{\phi}^{\mathrm{batch},\gamma})^\top
    \widetilde M_{\phi}^{\mathrm{batch},\gamma}
    v
    +
    \varepsilon_{\mathrm{spec}}
    } .
    \]
    \IF{$\beta>0$}
        \STATE Sample a temporary local perturbation $u\sim\mathcal{N}(0,\sigma_Q^2 I)$ on the trainable coordinates and set $\widetilde\phi=\phi+u$.
        \STATE Compute the stability penalty $\mathcal{R}_{\mathrm{stab}}$ via Eq.~\eqref{eq:R_stab_def}.
    \ELSE
        \STATE Set $\mathcal{R}_{\mathrm{stab}}=0$.
    \ENDIF
    \STATE Update $\phi$ by minimizing
    \[
    \mathcal{L}_{\mathrm{base}}
    +
    \alpha \mathcal{R}_{\mathrm{spec}}
    +
    \beta \mathcal{R}_{\mathrm{stab}} .
    \]
\ENDFOR
\end{algorithmic}
\end{algorithm}

\subsection{Gradient estimator induced by power iteration}
\label{app:power_iter_grad}
Let $A=\widetilde M^\gamma_\phi$ and define $B=A^\top A$.
Power iteration produces a unit vector $v$ approximating the top eigenvector of $B$, hence $\sigma_{\max}(A)\approx \sqrt{v^\top Bv}$.
If we treat $v$ as fixed within a refresh window, then
\begin{equation}
\widehat{\sigma}(A):=\sqrt{v^\top A^\top Av}
\quad\Rightarrow\quad
\frac{\partial \widehat{\sigma}}{\partial A}
=
\frac{1}{\widehat{\sigma}(A)}\,Avv^\top.
\label{eq:grad_sigma_hat}
\end{equation}

In implementation, we use the stabilized estimator
\[
\widehat\sigma_\varepsilon(A)
=
\sqrt{v^\top A^\top A v+\varepsilon_{\mathrm{spec}}},
\]
with a small $\varepsilon_{\mathrm{spec}}>0$.
The corresponding gradient is
\[
\frac{\partial \widehat\sigma_\varepsilon}{\partial A}
=
\frac{A vv^\top}{\widehat\sigma_\varepsilon(A)}.
\]

\subsection{Safety gate temperature}
\label{app:gate_temp}
The temperature $\kappa$ in Eq.~\eqref{eq:gate_sigmoid_prelim} interpolates between a hard margin indicator and a smooth weighting.
A smaller $\kappa$ makes $g_{\gamma,\kappa}^{\theta}$ closer to $\mathbf{1}\{\Delta\le \gamma\}$, but may increase gradient variance.

\subsection{Local Interpretation of the Joint Stability Regularizer}
\label{app:stab_local_interpretation}

We provide a local interpretation of the stability regularizer in Eq.~\eqref{eq:R_stab_def}. 
For a mini-batch index $q$, let
\[
z_q := s_{\theta(w)}(x_q),
\qquad
z'_{q,u} := s_{\theta(w+u)}(x'_q),
\]
where $u$ denotes a temporary perturbation in the trainable-coordinate space. 
Let
\[
p_q := \mathrm{softmax}(z_q/T),
\qquad
p'_{q,u} := \mathrm{softmax}(z'_{q,u}/T).
\]
The consistency loss compares $p_q$ with $p'_{q,u}$. 
Since the softmax is invariant to adding the same constant to all logits, we consider the centered logit drift
\[
\Delta z_q(u)
:=
\Pi\left(z'_{q,u}-z_q\right),
\qquad
\Pi
:=
I-\frac{1}{K}\mathbf{1}\mathbf{1}^{\top},
\]
where $\mathbf{1}\in\mathbb{R}^K$ is the all-one vector. 
% The projection $\Pi$ removes the shift-invariant logit component.
The projection $\Pi$ removes the common-shift direction, i.e., the component of the logits to which the softmax distribution is invariant.

For small centered logit drift, the softmax KL admits the second-order expansion
\[
\mathrm{KL}(p_q\|p'_{q,u})
=
\frac{1}{2T^2}
\Delta z_q(u)^{\top}
F_q
\Delta z_q(u)
+
O(\|\Delta z_q(u)\|^3),
\]
where
\[
F_q
:=
\mathrm{Diag}(p_q)-p_qp_q^{\top}
\]
is the Fisher matrix of the categorical distribution induced by the clean branch. 
Thus the KL consistency term locally penalizes a Fisher-weighted centered logit drift.

To separate the sources of this drift, assume $\tilde w=w+u$ with 
$u\sim\mathcal{N}(0,\sigma_Q^2I)$ and linearize the logits around $w$:
\[
s_{\theta(w+u)}(x'_q)
\approx
s_{\theta(w)}(x'_q)
+
\nabla_w s_{\theta(w)}(x'_q)u.
\]
Therefore,
\[
\Delta z_q(u)
\approx
a_q+J_qu,
\]
where
\[
a_q
:=
\Pi\left(
s_{\theta(w)}(x'_q)-s_{\theta(w)}(x_q)
\right)
\]
captures the clean-to-perturbed input drift, and
\[
J_q
:=
\Pi\nabla_w s_{\theta(w)}(x'_q)
\]
captures the local sensitivity of the logits to trainable-coordinate perturbations. 
Substituting this decomposition into the quadratic approximation gives
\[
(a_q+J_qu)^\top F_q(a_q+J_qu)
=
a_q^\top F_q a_q
+
2a_q^\top F_qJ_qu
+
u^\top J_q^\top F_qJ_qu.
\]
Taking expectation over the zero-mean Gaussian perturbation eliminates the cross term:
\[
\mathbb{E}_u[2a_q^\top F_qJ_qu]
=
2a_q^\top F_qJ_q\mathbb{E}[u]
=
0.
\]
Moreover, using $\mathbb{E}[uu^\top]=\sigma_Q^2I$, we have
\[
\mathbb{E}_u
\left[
u^\top J_q^\top F_qJ_qu
\right]
=
\sigma_Q^2
\mathrm{Tr}(J_q^\top F_qJ_q).
\]
Hence,
\[
\mathbb{E}_u\,
\mathrm{KL}(p_q\|p'_{q,u})
\approx
\frac{1}{2T^2}
a_q^\top F_q a_q
+
\frac{\sigma_Q^2}{2T^2}
\mathrm{Tr}(J_q^\top F_qJ_q).
\]
The first term penalizes clean-to-perturbed input drift, matching the
consistency principle used in standard robust training. The second term
penalizes the output effect of local perturbations in the trainable
coordinates. This second term is the part directly related to the stability
event used in Proposition~\ref{prop:deterministic_bridge}: it discourages a
posterior sample $\theta(w+u)$ from inducing a large centered logit drift
relative to $\theta(w)$ on the perturbed input.

To see the connection to margin stability, let
$d_q(u):=\Pi(s_{\theta(w+u)}(x'_q)-s_{\theta(w)}(x'_q))$. Since margins are
invariant to common logit shifts, only this centered drift matters. For any
label $y$,
\[
\left|
\Delta_{\theta(w+u)}(x'_q,y)
-
\Delta_{\theta(w)}(x'_q,y)
\right|
\le
2\|d_q(u)\|_{\infty}
\le
2\|d_q(u)\|_2 .
\]
Thus, a posterior-induced margin shift larger than $\gamma$ requires
$\|d_q(u)\|_2>\gamma/2$. Under the local non-degeneracy condition that the
Fisher matrix has eigenvalue at least $\lambda_q>0$ on the relevant centered
subspace, the quadratic term satisfies
\[
d_q(u)^\top F_q d_q(u)
\ge
\lambda_q \|d_q(u)\|_2^2 .
\]
% Consequently, by Markov's inequality, the probability of a large local margin
% shift is controlled by the expected Fisher-weighted drift:
% \[
% \Pr_u
% \!\left(
% \left|
% \Delta_{\theta(w+u)}(x'_q,y)
% -
% \Delta_{\theta(w)}(x'_q,y)
% \right|
% >
% \gamma
% \right)
% \lesssim
% \frac{4}{\lambda_q\gamma^2}
% \mathbb{E}_u
% \!\left[
% d_q(u)^\top F_q d_q(u)
% \right],
% \]
Consequently, by Markov's inequality, the probability of a large local margin shift in absolute value is controlled by the expected Fisher-weighted drift:
\[
\Pr_u\left(
\left|
\Delta_{\theta(w+u)}(x'_q,y)
-
\Delta_{\theta(w)}(x'_q,y)
\right|
>
\gamma
\right)
\lesssim
\frac{4}{\lambda_q\gamma^2}
\mathbb{E}_u
\left[
d_q(u)^\top F_q d_q(u)
\right],
\]
up to the local second-order approximation above. Therefore,
$\mathcal{R}_{\mathrm{stab}}$ should be understood as a practical proxy for
reducing the stability-failure probability $\rho$: it penalizes the input
drift term and, more importantly for the deterministic bridge, the
coordinate-noise sensitivity that can produce large margin shifts. The
PAC-Bayes complexity term controls the size of the posterior in parameter
space, while this Fisher-weighted term controls how such local perturbations
affect the predictive distribution.

\section{Special Cases of the Trainable-Coordinate PAC-Bayes Setup}
\label{app:special_cases}

\subsection{LoRA specialization}
In the LoRA case, the trainable coordinates are exactly the LoRA parameters:
\[
w=\phi_{\mathrm{LoRA}}\in\mathbb{R}^{d_\ell},
\qquad
d_{\mathrm{train}}=d_\ell,
\qquad
\theta(w)=\mathcal{T}_{\mathrm{LoRA}}(\theta_0,w).
\]
Under the Gaussian family
\begin{equation}
P=\mathcal{N}(0,\tau_P^2 I),
\qquad
Q=\mathcal{N}(\mu,\sigma_Q^2 I),
\label{eq:gaussian_PQ}
\end{equation}
the KL term becomes
\begin{equation}
\mathrm{KL}(Q\|P)
=
\frac{\|\mu\|_2^2}{2\tau_P^2}
+
\frac{d_\ell}{2}
\left(
\frac{\sigma_Q^2}{\tau_P^2}
-1
-\ln\frac{\sigma_Q^2}{\tau_P^2}
\right).
\label{eq:kl_gaussian}
\end{equation}
This is the LoRA-dimensional specialization of the Gaussian
$\mathrm{KL}(Q\|P)$ term used in Theorem~\ref{thm:pacbayes_spectral}.
The PAC-Bayes theorem itself is unchanged; only the instantiation of $w$ and $d_{\mathrm{train}}$ changes.

\subsection{Full-parameter fine-tuning}
\label{app:fullparam_pac}
For full fine-tuning, we take
\[
w=\theta-\theta_0\in\mathbb{R}^{d_{\mathrm{full}}},
\qquad
d_{\mathrm{train}}=d_{\mathrm{full}}.
\]
The theorem in Section~\ref{sec:main_theorem} again applies in exactly the same form.
The main difference is quantitative: the complexity term is typically larger because the Gaussian KL now scales with $d_{\mathrm{full}}$ rather than $d_\ell$.
The routes below are only \emph{sufficient} ways to instantiate the stability condition~\ref{prop:deterministic_bridge}; they are not automatic guarantees.

\paragraph{Setup.}
Let $\mu\in\mathbb{R}^{d_{\mathrm{full}}}$ denote the full fine-tuning displacement, so that the posterior mean corresponds to parameters $\theta(\mu)=\theta_0+\mu$.
For $u\sim\mathcal{N}(0,\sigma_Q^2 I_{d_{\mathrm{full}}})$, the sampled parameter vector is $\theta(\mu+u)$.
We seek sufficient conditions implying
\[
\Pr_{u\sim\mathcal{N}(0,\sigma_Q^2 I_{d_{\mathrm{full}}})}
\big(
\Xi(\mu,u)\le \gamma/2
\big)
\ge 1-\rho.
\]

\subsubsection*{Route 1: A Lipschitz-type bound (explicit but conservative)}
Suppose one can bound the score perturbation by
\[
\Xi(\mu,u)\le L_{\mu}\|u\|_2
\]
for some constant $L_\mu>0$.
For Transformers, such a bound can be obtained conservatively by combining layerwise operator-norm perturbation bounds with bounded hidden representations.
For example, if a pointwise nonlinearity $\psi$ is $L_\psi$-Lipschitz and the relevant linear maps admit perturbation bounds of the form
\[
\|(W+\Delta W)h-Wh\|_2\le \|\Delta W\|_2\,\|h\|_2,
\]
then repeated application through the network yields a global constant $L_\mu$.
Combining this with Gaussian norm concentration,
\[
\Pr\Big(
\|u\|_2 \le \sigma_Q\big(\sqrt{d_{\mathrm{full}}}+\sqrt{2\ln(1/\rho)}\big)
\Big)\ge 1-\rho,
\]
gives the sufficient condition
\[
L_\mu\,
\sigma_Q\big(\sqrt{d_{\mathrm{full}}}+\sqrt{2\ln(1/\rho)}\big)
\le \gamma/2.
\]

\subsubsection*{Route 2: Local (Jacobian-based) stability certificate}
Alternatively, one can use local first-order sensitivity.
If $s_{\theta(w)}(x',k)$ is differentiable in $w$, then for $u$ in a local neighborhood,
\[
|s_{\theta(\mu+u)}(x',k)-s_{\theta(\mu)}(x',k)|
\le
\|\nabla_w s_{\theta(\mu)}(x',k)\|_2\,\|u\|_2
+
\frac{H}{2}\|u\|_2^2,
\]
where $H$ controls the local Hessian along the segment $\mu+t u$.
Ignoring or separately controlling the second-order term yields the practical approximation
\[
\Xi(\mu,u)
\lesssim
G_\mu\,\|u\|_2,
\qquad
G_\mu
:=
\sup_{(x',y)\sim D'}\max_{k\in[K]}
\|\nabla_w s_{\theta(\mu)}(x',k)\|_2.
\]
Together with Gaussian norm concentration, a sufficient local condition is
\[
G_\mu\,
\sigma_Q\big(\sqrt{d_{\mathrm{full}}}+\sqrt{2\ln(1/\rho)}\big)
\le \gamma/2.
\]

\paragraph{Remark on complexity.}
Under a fully isotropic Gaussian posterior, the KL term scales with $d_{\mathrm{full}}$.
This is the main reason why the LoRA specialization is often more attractive in PAC-Bayes analyses, even though the theorem itself does not require parameter efficiency.

\section{Discussions and Limitations}
\label{sec:discussions_and_limitations}
\paragraph{Adaptive perturbation generation.}
Our PAC-Bayes analysis conditions on a fixed perturbed sample or a perturbation protocol independent of the learned posterior.
This is appropriate for the LLM text-perturbation setting and for frozen or sample-split perturbation generators.
The VLM PGD training loop is more adaptive: adversarial examples can depend on the current model parameters.
We therefore interpret the VLM results as an empirical stress test of the FragileFlow regularizer rather than as a direct instantiation of the strict PAC-Bayes theorem.
Extending the bound to fully model-dependent adversarial data generation, for example through data-dependent PAC-Bayes or algorithm-dependent perturbation kernels, is a direction for future work.

\paragraph{Trainable-coordinate choice and complexity.}
Our theory is stated over generic trainable coordinates and is not tied to LoRA. We use LoRA in the experiments mainly for resource efficiency and because it matches common adaptation practice for large models. Full-parameter fine-tuning is compatible with the same formal framework and may give the regularizer more freedom to reshape fragile error-flow patterns, but it also increases the PAC-Bayes complexity term and may require stronger control of the adapted parameters. In practice, this suggests a trade-off between adaptation capacity and complexity control; empirical parameter regularization, structured posteriors, or norm-constrained updates are natural ways to keep this trade-off stable. We leave a systematic comparison between LoRA and full-parameter adaptation to future work.

\paragraph{Calibration and tuning budget.}
Our experiments use a fixed calibration protocol for the safety buffer and a limited sensitivity sweep for the plug-in weights, rather than exhaustively tuning these choices for every model--dataset--learner combination. This makes the comparisons more conservative and avoids selecting hyperparameters directly to maximize each reported test metric, but it also means that the reported numbers should not be read as the best achievable performance of FragileFlow. The consistent reductions in vulnerable-flow measures under this restrained tuning protocol support the main theory-facing claim, while larger-scale calibration and task-specific tuning may further improve the robustness--utility trade-off.

\paragraph{Data scale and perturbation strength.}
We evaluate FragileFlow under fixed data budgets and fixed perturbation protocols in order to keep the comparison controlled across models, datasets, and base learners. We do not exhaustively study how the effect changes with larger adaptation sets, different calibration-set sizes, or stronger perturbation budgets. These factors may influence both the estimated vulnerable-flow matrix and the robustness--utility trade-off. A more systematic scaling study over data size and perturbation strength would be useful for understanding when spectral error-flow control is most beneficial.

\paragraph{Broader impacts.}
This work studies a general robustness objective and does not introduce a new deployment domain or user-facing system. Its positive impact is that worst-class-oriented robustness may help reduce concentrated failures in finite-option LLM and VLM applications, especially when errors repeatedly affect a small subset of classes or choices. At the same time, robustness methods are dual-use: stronger adaptation techniques could also make undesirable or poorly governed models more stable under perturbation, and improved robustness metrics may create overconfidence if used without task-specific safety evaluation. For this reason, FragileFlow should be viewed as a diagnostic and training tool to be combined with domain-specific validation, safety testing, and monitoring rather than as a standalone guarantee of safe deployment.
\section{Additional Experimental Results}
\label{app:exp}

% TODO: narrative to be added. For now this file only wires in the figures and tables
% in the intended order of appearance.

\subsection{LLM main results with per-cell standard deviations}
\label{app:llm_main_grid_std}

    \begin{table}[h]
    \centering
    \scriptsize
    \setlength{\tabcolsep}{2pt}
    \caption{\textbf{LLM main results (mean $\pm$ seed std): Clean Worst-Class Acc (\%).} This is the WC-Acc on clean data excluding the perturbation. Reported as mean $\pm$ seed standard deviation over three paired seeds. calibrated buffer $\gamma_{25}$, default $(\alpha,\beta)$.}
    \label{tab:llm_main_std_wc}
    \resizebox{\textwidth}{!}{%
    \begin{tabular}{llccccccccc}
    \toprule
     & & & \multicolumn{2}{c}{CE+aug} & \multicolumn{2}{c}{R3F} & \multicolumn{2}{c}{SMART}  \\
    \cmidrule(lr){4-5}\cmidrule(lr){6-7}\cmidrule(lr){8-9}
    Model & Dataset & \emph{CE} & base & +plug & base & +plug & base & +plug  \\
    \midrule
    Qwen-0.5B & ARC-C & \emph{41.16 $\pm$ 3.28} & 44.18 $\pm$ 3.98 & 43.47 $\pm$ 3.93 & 40.89 $\pm$ 2.83 & 42.84 $\pm$ 5.22 & 34.13 $\pm$ 3.79 & 43.38 $\pm$ 4.38  \\
    Qwen-0.5B & CSQA & \emph{45.22 $\pm$ 3.64} & 45.33 $\pm$ 7.16 & 45.56 $\pm$ 7.94 & 45.56 $\pm$ 3.37 & 45.00 $\pm$ 3.36 & 48.00 $\pm$ 4.01 & 48.89 $\pm$ 3.69  \\
    Qwen-1.5B & ARC-C & \emph{66.84 $\pm$ 2.85} & 68.27 $\pm$ 5.43 & 70.67 $\pm$ 2.49 & 66.67 $\pm$ 3.01 & 66.67 $\pm$ 2.89 & 66.49 $\pm$ 3.18 & 66.76 $\pm$ 4.36  \\
    Qwen-1.5B & CSQA & \emph{61.67 $\pm$ 2.79} & 67.56 $\pm$ 2.10 & 67.33 $\pm$ 1.80 & 61.78 $\pm$ 2.72 & 61.67 $\pm$ 2.87 & 63.22 $\pm$ 3.19 & 63.33 $\pm$ 3.87  \\
    Mistral-7B & ARC-C & \emph{67.29 $\pm$ 4.87} & 65.69 $\pm$ 4.06 & 68.53 $\pm$ 6.25 & 67.47 $\pm$ 4.41 & 67.11 $\pm$ 4.03 & 42.13 $\pm$ 20.54 & 54.84 $\pm$ 14.82  \\
    Mistral-7B & CSQA & \emph{66.00 $\pm$ 3.62} & 67.11 $\pm$ 3.88 & 66.89 $\pm$ 3.57 & 66.67 $\pm$ 4.44 & 66.44 $\pm$ 4.07 & 50.56 $\pm$ 4.87 & 57.89 $\pm$ 3.42  \\
    \bottomrule
    \end{tabular}}
    \end{table}

    \begin{table}[h]
    \centering
    \scriptsize
    \setlength{\tabcolsep}{2pt}
    \caption{\textbf{LLM main results (mean $\pm$ seed std): Accuracy on perturbation-only inputs (\%).} Each cell averages accuracy over $\{$typo, distractor, format\_rewrite$\}$. calibrated buffer $\gamma_{25}$, default $(\alpha,\beta)$. Reported as mean $\pm$ seed standard deviation over three paired seeds.}
    \label{tab:llm_main_std_acc_pert}
    \resizebox{\textwidth}{!}{%
    \begin{tabular}{llccccccccc}
    \toprule
     & & & \multicolumn{2}{c}{CE+aug} & \multicolumn{2}{c}{R3F} & \multicolumn{2}{c}{SMART}  \\
    \cmidrule(lr){4-5}\cmidrule(lr){6-7}\cmidrule(lr){8-9}
    Model & Dataset & \emph{CE} & base & +plug & base & +plug & base & +plug  \\
    \midrule
    Qwen-0.5B & ARC-C & \emph{41.44 $\pm$ 1.80} & 49.16 $\pm$ 6.02 & 49.47 $\pm$ 5.50 & 41.53 $\pm$ 1.97 & 42.40 $\pm$ 2.36 & 42.16 $\pm$ 1.28 & 42.11 $\pm$ 0.99  \\
    Qwen-0.5B & CSQA & \emph{42.82 $\pm$ 1.77} & 53.60 $\pm$ 2.53 & 54.36 $\pm$ 2.68 & 42.67 $\pm$ 1.64 & 42.73 $\pm$ 1.50 & 44.67 $\pm$ 2.08 & 44.16 $\pm$ 2.31  \\
    Qwen-1.5B & ARC-C & \emph{62.80 $\pm$ 3.00} & 72.18 $\pm$ 1.16 & 73.22 $\pm$ 1.10 & 62.84 $\pm$ 3.10 & 63.02 $\pm$ 3.38 & 68.47 $\pm$ 2.33 & 68.20 $\pm$ 2.59  \\
    Qwen-1.5B & CSQA & \emph{59.98 $\pm$ 3.12} & 71.49 $\pm$ 1.37 & 71.53 $\pm$ 1.08 & 59.98 $\pm$ 3.18 & 59.71 $\pm$ 3.33 & 65.33 $\pm$ 1.80 & 65.58 $\pm$ 3.52  \\
    Mistral-7B & ARC-C & \emph{66.78 $\pm$ 2.12} & 72.42 $\pm$ 2.85 & 73.04 $\pm$ 4.04 & 66.36 $\pm$ 2.34 & 66.07 $\pm$ 2.45 & 55.51 $\pm$ 5.35 & 56.98 $\pm$ 4.78  \\
    Mistral-7B & CSQA & \emph{66.42 $\pm$ 3.44} & 74.11 $\pm$ 2.42 & 74.60 $\pm$ 2.19 & 66.47 $\pm$ 2.96 & 66.11 $\pm$ 3.10 & 51.96 $\pm$ 4.03 & 52.89 $\pm$ 4.34  \\
    \bottomrule
    \end{tabular}}
    \end{table}

\paragraph{Discussion.}
These auxiliary metrics are reported mainly for completeness and as utility checks, aligning the LLM appendix with the additional metrics reported for VLMs. 
Clean worst-class accuracy shows mixed but generally small changes across base learners, while perturbation-only average accuracy is mostly stable. 
Thus, these results do not form a separate claim of uniform improvement; they simply indicate that the vulnerable-flow reductions reported in the main paper are not obtained by a clear degradation of standard utility metrics.
% \newpage

% \subsection{Per-model buffer sweep (LLM)}
% \label{app:llm_gamma_sweep}

% \begin{figure}[h]
%     \centering
%     \includegraphics[width=\textwidth]{FIGURE_paper/fig_gamma_per_model_grid_v5_no_awp_clean_accuracy_pert_safety_with_tables.pdf}
%     \caption{\textbf{Per-model $\gamma$ sweep (LLM).} Full grid of paired base (dashed) vs.\ base $+$ plug-in (solid) curves
%     for every (model, dataset, base learner) triple in the beneficial regime. Companion to the aggregated main-paper figure.}
%     \label{fig:gamma_per_model_grid}
% \end{figure}

% \newpage

\subsection{VLM 4-shot low-shot stress test}
\label{app:vlm_vit_shot4}

% CLIP ViT-B/16 4-shot low-shot stress test (companion to the 16-shot main table).

\begin{table}[h]
\centering
\small
\setlength{\tabcolsep}{3pt}
\caption{\textbf{CLIP ViT-B/32 4-shot low-shot stress test.} Same evaluation protocol as the main ViT table, but with shots = 4. Robustness is evaluated on test images using the default PGD setting of our ViT pipeline (\texttt{attack\_type=PGD}, $\varepsilon=1.0/255$, 100 attack steps). PGD WC-Acc denotes worst-class accuracy on PGD-perturbed test images. \textbf{Bold = best among LoRA-adv family rows}, as in the main table.}
\label{tab:vlm_vit_shot4}
\resizebox{\textwidth}{!}{%
\begin{tabular}{llccccccc}
\toprule
Dataset & Method & Clean Acc $\uparrow$ & PGD Acc $\uparrow$ & Clean WC $\uparrow$ & PGD WC $\uparrow$ & $\widehat{\mathrm{VSR}}_{\gamma}$ (pgd) $\downarrow$ & $\mathrm{VWR}_\gamma$ (pgd) $\downarrow$ & $n$ \\
\midrule
\multirow{4}{*}{DTD}
 % & ZeroShot                 & $42.79$ & $14.60$ & $0.00$ & $0.00$ & $18.83 \pm 0.35$ & $21.39 \pm 0.33$ & 3 \\
 % & LoRA                     & $60.26 \pm 0.97$ & $20.15 \pm 2.28$ & $23.15 \pm 3.46$ & $0.00$ & $37.25 \pm 3.67$ & $27.10 \pm 0.18$ & 3 \\
 % & LoRA + plugin            & $60.15 \pm 1.02$ & $19.98 \pm 2.01$ & $23.15 \pm 3.46$ & $0.00$ & $36.60 \pm 4.03$ & $27.05 \pm 0.33$ & 3 \\
 & LoRA-adv                 & $59.54 \pm 0.37$ & $20.23 \pm 1.19$ & $\mathbf{15.74 \pm 5.71}$ & $\mathbf{0.00}$ & $32.19 \pm 2.83$ & $26.92 \pm 0.06$ & 3 \\
 & LoRA-adv + plugin (inner) & $\mathbf{59.57 \pm 0.49}$ & $20.37 \pm 1.24$ & $15.74 \pm 5.71$ & $0.00$ & $32.37 \pm 3.03$ & $27.01 \pm 0.12$ & 3 \\
 & LoRA-adv + plugin (outer) & $58.94 \pm 0.85$ & $\mathbf{20.41 \pm 1.21}$ & $12.96 \pm 7.29$ & $0.00$ & $31.21 \pm 1.95$ & $26.66 \pm 0.28$ & 3 \\
 & LoRA-adv + plugin (both) & $59.06 \pm 0.91$ & $20.29 \pm 1.23$ & $12.04 \pm 6.55$ & $0.00$ & $\mathbf{31.18 \pm 2.41}$ & $\mathbf{26.61 \pm 0.35}$ & 3 \\
\midrule
\multirow{4}{*}{OxfordPets}
 % & ZeroShot                 & $85.04$ & $21.42$ & $0.00$ & $0.00$ & $45.45 \pm 1.51$ & $48.85 \pm 1.49$ & 3 \\
 % & LoRA                     & $86.94$ & $19.24$ & $56.00$ & $1.00$ & $75.55$ & $66.11$ & 1 \\
 % & LoRA + plugin            & $87.41$ & $19.57$ & $56.00$ & $1.00$ & $69.53$ & $64.72$ & 1 \\
 & LoRA-adv                 & $86.64 \pm 0.44$ & $19.48 \pm 0.72$ & $54.33 \pm 9.29$ & $0.67 \pm 0.47$ & $57.09 \pm 1.99$ & $61.77 \pm 0.36$ & 3 \\
 & LoRA-adv + plugin (inner) & $86.73 \pm 0.53$ & $19.49 \pm 0.70$ & $54.33 \pm 10.50$ & $0.67 \pm 0.47$ & $57.32 \pm 2.41$ & $61.48 \pm 0.48$ & 3 \\
 & LoRA-adv + plugin (outer) & $\mathbf{87.27}$ & $\mathbf{20.55}$ & $\mathbf{56.00}$ & $\mathbf{1.01}$ & $\mathbf{53.48}$ & $60.84$ & 1 \\
 & LoRA-adv + plugin (both) & $86.98 \pm 0.42$ & $19.88 \pm 0.78$ & $53.00 \pm 11.43$ & $1.00 \pm 0.82$ & $55.25 \pm 1.80$ & $\mathbf{60.31 \pm 0.24}$ & 3 \\
\midrule
\multirow{4}{*}{Caltech101}
 % & ZeroShot                 & $91.40$ & $56.43$ & $6.67$ & $0.00$ & $19.54 \pm 0.34$ & $56.19 \pm 1.25$ & 3 \\
 % & LoRA                     & $93.96 \pm 0.29$ & $60.14 \pm 0.86$ & $31.43 \pm 12.12$ & $0.00$ & $23.80 \pm 3.04$ & $62.04 \pm 1.01$ & 3 \\
 % & LoRA + plugin            & $94.00 \pm 0.35$ & $60.30 \pm 0.96$ & $33.65 \pm 13.96$ & $0.00$ & $23.25 \pm 4.26$ & $60.57 \pm 0.74$ & 3 \\
 & LoRA-adv                 & $94.04 \pm 0.03$ & $\mathbf{60.97 \pm 1.06}$ & $25.16 \pm 14.13$ & $\mathbf{1.28 \pm 1.81}$ & $25.71 \pm 4.56$ & $55.83 \pm 3.56$ & 3 \\
 & LoRA-adv + plugin (inner) & $94.04 \pm 0.10$ & $60.82 \pm 1.06$ & $\mathbf{26.46 \pm 13.84}$ & $1.28 \pm 1.81$ & $25.63 \pm 4.62$ & $56.09 \pm 3.45$ & 3 \\
 & LoRA-adv + plugin (outer) & $\mathbf{94.08 \pm 0.09}$ & $60.85 \pm 0.98$ & $24.23 \pm 13.07$ & $0.00$ & $24.26 \pm 4.03$ & $54.62 \pm 3.15$ & 3 \\
 & LoRA-adv + plugin (both) & $93.98 \pm 0.11$ & $60.77 \pm 0.89$ & $24.23 \pm 13.07$ & $0.00$ & $\mathbf{24.04 \pm 4.58}$ & $\mathbf{54.45 \pm 3.03}$ & 3 \\
\bottomrule
\end{tabular}}
\end{table}

\paragraph{Discussion.}
This 4-shot setting is included as a low-shot stress test and as the per-dataset detail behind the VLM ablation in Table~\ref{tab:vlm_b0_ablation}. 
Because the adaptation set is very small, the results are naturally noisy and should not be read as a claim that every plug-in placement improves every metric. 
The main signal is that the outer and both variants often reduce the vulnerable-flow measures while keeping clean and PGD average accuracy close to the LoRA-adv baseline, whereas the inner-only variant is less stable. 
PGD worst-class accuracy is frequently floor-limited in this setting, so it provides limited resolution beyond showing the difficulty of the stress test. 
We therefore use this table as supporting evidence that FragileFlow remains useful in a harder low-shot protocol, rather than as a primary claim of uniform improvement.

% \newpage

\subsection{\texorpdfstring{$\beta = 0$}{beta = 0} ablation: per-cell numeric breakdown}
\label{app:beta_ablation_numbers}

% LLM beta=0 ablation: R_spec only vs composite
% Base rows: 3 seeds on both ARC and CSQA (s7, s42, s123).
% +Plug (beta=0) rows: 2 seeds on both datasets (s42, s123) -- the beta=0 batch did not include s7.
% +Plug (beta>0) rows: 3 seeds at gamma=0.25 after the s7 ARC rerun (s7, s42, s123).
%without AWP!!!

\begin{table}[h]
\centering
\small
\setlength{\tabcolsep}{4pt}
\caption{\textbf{LLM $\beta=0$ ablation ($\mathcal{R}_{\mathrm{spec}}$ only vs.\ composite).}
This table is intended as a mechanism ablation for the stability term.
The main comparison is between the spectral-only plug-in $(\beta=0)$ and the composite objective $(\beta>0)$.
\textbf{Base} provides a non-plug scale reference averaged over four learners:
\texttt{base\_clean} (CE), \texttt{base\_aug}, \texttt{r3f}, and \texttt{smart}.
The +Plug rows are computed from the corresponding plug-in runs.
Values in parentheses report changes relative to \textbf{Base}; the bottom block averages these shifts over the six model--dataset cells.
Lower is better for $\widehat{\mathrm{VSR}}_\gamma$ and $\widehat{\mathrm{VWR}}_\gamma$, while higher is better for accuracy metrics.
Bold marks the desired direction relative to \textbf{Base}.}
\label{tab:llm_b0_ablation}
\resizebox{0.9\textwidth}{!}{%
\begin{tabular}{llcrrrrr}
\toprule
Model & Dataset & Config
& $\widehat{\mathrm{VSR}}_\gamma$ $\downarrow$
& $\widehat{\mathrm{VWR}}_\gamma$ $\downarrow$
& WC-Acc $\uparrow$
& Ptb Acc $\uparrow$
& Clean Acc $\uparrow$ \\
\midrule
\multirow{3}{*}{Qwen-0.5B} & \multirow{3}{*}{ARC-C} & Base & $45.06$ & $50.27$ & $40.09$ & $50.01$ & $52.67$ \\
 & & +Plug ($\beta{=}0$) & $\mathbf{34.88} (\mathbf{-10.18})$ & $\mathbf{38.44} (\mathbf{-11.83})$ & $\mathbf{43.69} (\mathbf{+3.60})$ & $\mathbf{51.47} (\mathbf{+1.46})$ & $\mathbf{54.30} (\mathbf{+1.63})$ \\
 & & +Plug ($\beta{>}0$) & $\mathbf{34.67} (\mathbf{-10.39})$ & $\mathbf{38.25} (\mathbf{-12.02})$ & $\mathbf{43.23} (\mathbf{+3.14})$ & $\mathbf{50.90} (\mathbf{+0.89})$ & $\mathbf{53.64} (\mathbf{+0.98})$ \\
\cmidrule(lr){2-8}
\multirow{3}{*}{Qwen-0.5B} & \multirow{3}{*}{CSQA} & Base & $35.62$ & $39.81$ & $46.03$ & $53.10$ & $58.50$ \\
 & & +Plug ($\beta{=}0$) & $\mathbf{26.79} (\mathbf{-8.84})$ & $\mathbf{30.42} (\mathbf{-9.39})$ & $\mathbf{46.09} (\mathbf{+0.06})$ & $\mathbf{53.84} (\mathbf{+0.74})$ & $\mathbf{59.32} (\mathbf{+0.82})$ \\
 & & +Plug ($\beta{>}0$) & $\mathbf{27.27} (\mathbf{-8.35})$ & $\mathbf{30.93} (\mathbf{-8.88})$ & $\mathbf{46.48} (\mathbf{+0.45})$ & $\mathbf{53.80} (\mathbf{+0.70})$ & $\mathbf{59.42} (\mathbf{+0.92})$ \\
\cmidrule(lr){2-8}
\multirow{3}{*}{Qwen-1.5B} & \multirow{3}{*}{ARC-C} & Base & $27.37$ & $30.94$ & $67.07$ & $71.34$ & $74.92$ \\
 & & +Plug ($\beta{=}0$) & $\mathbf{25.36} (\mathbf{-2.01})$ & $\mathbf{28.86} (\mathbf{-2.08})$ & $66.40 (-0.67)$ & $\mathbf{71.73} (\mathbf{+0.39})$ & $\mathbf{75.00} (\mathbf{+0.08})$ \\
 & & +Plug ($\beta{>}0$) & $\mathbf{25.50} (\mathbf{-1.87})$ & $\mathbf{28.33} (\mathbf{-2.61})$ & $\mathbf{68.03} (\mathbf{+0.96})$ & $\mathbf{71.93} (\mathbf{+0.59})$ & $\mathbf{75.36} (\mathbf{+0.44})$ \\
\cmidrule(lr){2-8}
\multirow{3}{*}{Qwen-1.5B} & \multirow{3}{*}{CSQA} & Base & $23.93$ & $27.18$ & $63.56$ & $69.22$ & $74.23$ \\
 & & +Plug ($\beta{=}0$) & $\mathbf{20.52} (\mathbf{-3.41})$ & $\mathbf{23.13} (\mathbf{-4.04})$ & $\mathbf{65.22} (\mathbf{+1.67})$ & $\mathbf{70.05} (\mathbf{+0.83})$ & $\mathbf{74.96} (\mathbf{+0.72})$ \\
 & & +Plug ($\beta{>}0$) & $\mathbf{20.25} (\mathbf{-3.68})$ & $\mathbf{23.25} (\mathbf{-3.93})$ & $\mathbf{64.11} (\mathbf{+0.56})$ & $\mathbf{69.76} (\mathbf{+0.54})$ & $\mathbf{74.47} (\mathbf{+0.23})$ \\
\cmidrule(lr){2-8}
\multirow{3}{*}{Mistral-7B} & \multirow{3}{*}{ARC-C} & Base & $31.02$ & $37.02$ & $60.64$ & $70.40$ & $73.75$ \\
 & & +Plug ($\beta{=}0$) & $\mathbf{30.56} (\mathbf{-0.47})$ & $\mathbf{35.00} (\mathbf{-2.02})$ & $60.22 (-0.42)$ & $70.08 (-0.32)$ & $73.60 (-0.15)$ \\
 & & +Plug ($\beta{>}0$) & $\mathbf{29.62} (\mathbf{-1.40})$ & $\mathbf{34.73} (\mathbf{-2.29})$ & $\mathbf{63.50} (\mathbf{+2.85})$ & $\mathbf{71.41} (\mathbf{+1.01})$ & $\mathbf{74.58} (\mathbf{+0.83})$ \\
\cmidrule(lr){2-8}
\multirow{3}{*}{Mistral-7B} & \multirow{3}{*}{CSQA} & Base & $24.17$ & $29.09$ & $62.58$ & $70.82$ & $74.98$ \\
 & & +Plug ($\beta{=}0$) & $\mathbf{22.32} (\mathbf{-1.85})$ & $\mathbf{27.73} (\mathbf{-1.36})$ & $61.24 (-1.34)$ & $70.34 (-0.48)$ & $74.41 (-0.57)$ \\
 & & +Plug ($\beta{>}0$) & $\mathbf{21.60} (\mathbf{-2.57})$ & $\mathbf{25.91} (\mathbf{-3.18})$ & $\mathbf{63.74} (\mathbf{+1.16})$ & $\mathbf{71.19} (\mathbf{+0.37})$ & $\mathbf{75.00} (\mathbf{+0.02})$ \\
\midrule
\multicolumn{3}{l}{\emph{Averaged shift vs.\ Base (6 cells)}}
& $\Delta_{\mathrm{rel}}\,\widehat{\mathrm{VSR}}_\gamma$
& $\Delta_{\mathrm{rel}}\,\widehat{\mathrm{VWR}}_\gamma$
& $\Delta_{\mathrm{abs}}\,\mathrm{WC}$
& $\Delta_{\mathrm{abs}}\,\mathrm{PtbAcc}$
& $\Delta_{\mathrm{abs}}\,\mathrm{CleanAcc}$ \\
% \multicolumn{3}{l}{+Plug ($\beta=0$)}
% & $-12.83\%$ & $-13.38\%$ & $-0.00$ pp & $+0.32$ pp & $+0.30$ pp \\
% \multicolumn{3}{l}{+Plug ($\beta>0$)}
% & $-13.98\%$ & $-14.42\%$ & $+1.39$ pp & $+0.49$ pp & $+0.43$ pp \\
\multicolumn{3}{l}{+Plug ($\beta=0$)} & $-13.02\%$ & $-13.14\%$ & $+0.48$ pp & $+0.44$ pp & $+0.42$ pp \\
\multicolumn{3}{l}{+Plug ($\beta>0$)} & $-13.97\%$ & $-14.37\%$ & $+1.52$ pp & $+0.68$ pp & $+0.57$ pp \\
\bottomrule
\end{tabular}}
\end{table}

\begin{table}[h]
\centering
\small
\setlength{\tabcolsep}{4pt}
\caption{\textbf{VLM $\beta=0$ ablation ($\mathcal{R}_{\mathrm{spec}}$ only vs.\ composite).}
% This table reports a separate CLIP ViT-B/32 + LoRA mechanism ablation. Each plug-in cell averages over plug variants:
% % $\{$\texttt{lora\_plugin}, \texttt{inner}, \texttt{outer}, \texttt{both}$\}$. 
% $\{$\texttt{lora\_plugin}, \texttt{lora\_adv\_plugins}$\}$. 
% Base averages 
% $\{$\texttt{lora}, \texttt{lora\_adv}$\}$. 
% Ptb WC and Ptb Acc are measured under PGD evaluation, while Clean Acc is shown as a utility check.
% Desired direction: $\widehat{\mathrm{VSR}}_{\gamma}\downarrow$,
% $\widehat{\mathrm{VWR}}_{\gamma}\downarrow$, Ptb WC/Ptb Acc/Clean Acc $\uparrow$.
% \textbf{Bold} marks the desired direction relative to the matching Base row.
This table reports a CLIP ViT-B/32 + LoRA mechanism ablation for the stability term.
The main comparison is between the spectral-only plug-in $(\beta=0)$ and the composite objective $(\beta>0)$.
\textbf{Base} provides a non-plug scale reference averaged over \texttt{lora} and \texttt{lora\_adv}; the +Plug rows average the corresponding plug-in variants.
Values in parentheses report changes relative to \textbf{Base}, and the bottom block averages these shifts over the six dataset--shot cells.
Because PGD WC is often near the floor, this VLM ablation is mainly interpreted as evidence of vulnerable-risk compression and utility preservation.
Lower is better for $\widehat{\mathrm{VSR}}_\gamma$ and $\widehat{\mathrm{VWR}}_\gamma$, while higher is better for accuracy metrics.
Bold marks the desired direction relative to \textbf{Base}.}
\label{tab:vlm_b0_ablation}
\resizebox{0.9\textwidth}{!}{%
\begin{tabular}{llcrrrrr}
\toprule
Dataset & Shots & Config
& $\widehat{\mathrm{VSR}}_{\gamma}$ $\downarrow$
& $\widehat{\mathrm{VWR}}_{\gamma}$ $\downarrow$
& Ptb WC $\uparrow$
& Ptb Acc $\uparrow$
& Clean Acc $\uparrow$ \\
\midrule
\multicolumn{8}{l}{\emph{PGD evaluation} (Ptb WC/Ptb Acc = PGD worst-class/accuracy; Clean Acc shown in the last column)} \\
\midrule

DTD & 4 & Base
& $30.99$ & $27.05$ & $0.00$ & $20.92$ & $60.76$ \\
DTD & 4 & +Plug ($\beta{=}0$)
& $\mathbf{30.42}\ (\mathbf{-0.57})$
& $\mathbf{26.93}\ (\mathbf{-0.12})$
& $0.00\ (+0.00)$
& $\mathbf{21.04}\ (\mathbf{+0.12})$
& $60.25\ (-0.51)$ \\
DTD & 4 & +Plug ($\beta{>}0$)
& $\mathbf{28.76}\ (\mathbf{-2.23})$
& $\mathbf{25.78}\ (\mathbf{-1.28})$
& $0.00\ (+0.00)$
& $20.54\ (-0.38)$
& $60.71\ (-0.05)$ \\

DTD & 16 & Base
& $49.33$ & $27.82$ & $0.00$ & $21.37$ & $68.62$ \\
DTD & 16 & +Plug ($\beta{=}0$)
& $\mathbf{47.10}\ (\mathbf{-2.22})$
& $\mathbf{27.74}\ (\mathbf{-0.08})$
& $0.00\ (+0.00)$
& $21.25\ (-0.12)$
& $68.22\ (-0.40)$ \\
DTD & 16 & +Plug ($\beta{>}0$)
& $\mathbf{46.51}\ (\mathbf{-2.82})$
& $\mathbf{27.82}\ (\mathbf{-0.01})$
& $0.00\ (+0.00)$
& $21.11\ (-0.25)$
& $68.31\ (-0.31)$ \\

OxfordPets & 4 & Base
& $65.86$ & $64.19$ & $1.00$ & $19.34$ & $86.56$ \\
OxfordPets & 4 & +Plug ($\beta{=}0$)
& $\mathbf{58.40}\ (\mathbf{-7.45})$
& $\mathbf{62.33}\ (\mathbf{-1.86})$
& $\mathbf{1.01}\ (\mathbf{+0.01})$
& $\mathbf{19.78}\ (\mathbf{+0.44})$
& $\mathbf{86.87}\ (\mathbf{+0.31})$ \\
OxfordPets & 4 & +Plug ($\beta{>}0$)
& $\mathbf{58.11}\ (\mathbf{-7.75})$
& $\mathbf{62.07}\ (\mathbf{-2.12})$
& $\mathbf{1.25}\ (\mathbf{+0.25})$
& $\mathbf{19.94}\ (\mathbf{+0.60})$
& $\mathbf{86.90}\ (\mathbf{+0.34})$ \\

OxfordPets & 16 & Base
& $65.61$ & $67.10$ & $0.50$ & $18.70$ & $88.70$ \\
OxfordPets & 16 & +Plug ($\beta{=}0$)
& $\mathbf{62.56}\ (\mathbf{-3.05})$
& $\mathbf{66.18}\ (\mathbf{-0.93})$
& $0.25\ (-0.25)$
& $\mathbf{19.01}\ (\mathbf{+0.31})$
& $\mathbf{89.03}\ (\mathbf{+0.33})$ \\
OxfordPets & 16 & +Plug ($\beta{>}0$)
& $\mathbf{62.32}\ (\mathbf{-3.29})$
& $\mathbf{65.94}\ (\mathbf{-1.16})$
& $0.25\ (-0.25)$
& $\mathbf{19.20}\ (\mathbf{+0.50})$
& $\mathbf{88.98}\ (\mathbf{+0.28})$ \\

Caltech101 & 4 & Base
& $25.83$ & $55.99$ & $0.00$ & $61.89$ & $94.14$ \\
Caltech101 & 4 & +Plug ($\beta{=}0$)
& $\mathbf{24.72}\ (\mathbf{-1.10})$
& $\mathbf{53.33}\ (\mathbf{-2.66})$
& $0.00\ (+0.00)$
& $\mathbf{62.11}\ (\mathbf{+0.22})$
& $94.13\ (-0.01)$ \\
Caltech101 & 4 & +Plug ($\beta{>}0$)
& $\mathbf{24.79}\ (\mathbf{-1.04})$
& $\mathbf{53.12}\ (\mathbf{-2.87})$
& $0.00\ (+0.00)$
& $\mathbf{62.05}\ (\mathbf{+0.16})$
& $94.14\ (+0.00)$ \\

Caltech101 & 16 & Base
& $19.94$ & $48.44$ & $5.00$ & $65.13$ & $95.09$ \\
Caltech101 & 16 & +Plug ($\beta{=}0$)
& $\mathbf{19.02}\ (\mathbf{-0.93})$
& $\mathbf{43.44}\ (\mathbf{-5.00})$
& $\mathbf{6.67}\ (\mathbf{+1.67})$
& $\mathbf{66.15}\ (\mathbf{+1.01})$
& $94.90\ (-0.19)$ \\
Caltech101 & 16 & +Plug ($\beta{>}0$)
& $\mathbf{18.92}\ (\mathbf{-1.02})$
& $\mathbf{43.02}\ (\mathbf{-5.42})$
& $\mathbf{6.67}\ (\mathbf{+1.67})$
& $\mathbf{66.20}\ (\mathbf{+1.06})$
& $94.95\ (-0.14)$ \\

\midrule
\multicolumn{3}{l}{\emph{Averaged shift vs.\ Base (6 dataset $\times$ shot cells)}}
& $\Delta_{\mathrm{rel}}\,\widehat{\mathrm{VSR}}_{\gamma}$
& $\Delta_{\mathrm{rel}}\,\widehat{\mathrm{VWR}}_{\gamma}$
& $\Delta_{\mathrm{abs}}\,\mathrm{PtbWC}$
& $\Delta_{\mathrm{abs}}\,\mathrm{PtbAcc}$
& $\Delta_{\mathrm{abs}}\,\mathrm{CleanAcc}$ \\
\multicolumn{3}{l}{+Plug ($\beta=0$)}
& $-5.21\%$ & $-3.35\%$ & $+0.24$ pp & $+0.33$ pp & $-0.08$ pp \\
\multicolumn{3}{l}{+Plug ($\beta>0$)}
& $-6.47\%$ & $-4.35\%$ & $+0.28$ pp & $+0.28$ pp & $+0.02$ pp \\
\bottomrule
\end{tabular}}
\end{table}

\paragraph{Discussion.}
Table~\ref{tab:llm_b0_ablation} and Table~\ref{tab:vlm_b0_ablation} shows the functional ablation of the two terms we designed. The ablation isolates the role of the stability term in the composite plug-in objective.
The spectral-only variant already reduces the vulnerable-flow measures in both LLM and VLM settings, confirming that the main effect comes from controlling the class-structured error-flow matrix.
Adding the stability term further improves the average LLM shifts and makes the accuracy-side gains more reliable, especially in cases where the spectral-only variant reduces risk but does not consistently improve worst-class or perturbation accuracy.
For VLMs, the same pattern appears mainly through stronger vulnerable-risk compression and utility preservation, while PGD worst-class accuracy remains floor-limited and is therefore less informative.
Thus, these tables support the intended mechanism: spectral control is the core driver of vulnerable-flow reduction, and the stability term helps this control translate more stably into downstream robustness and utility.

\newpage

\subsection{\texorpdfstring{Plug-in strength sensitivity ($\alpha,\beta$ sweep)}{Plug-in strength sensitivity (alpha, beta sweep)}}
\label{app:alphabeta_sweep}

\begin{figure}[h]
    \centering
    \includegraphics[width=0.9\textwidth]{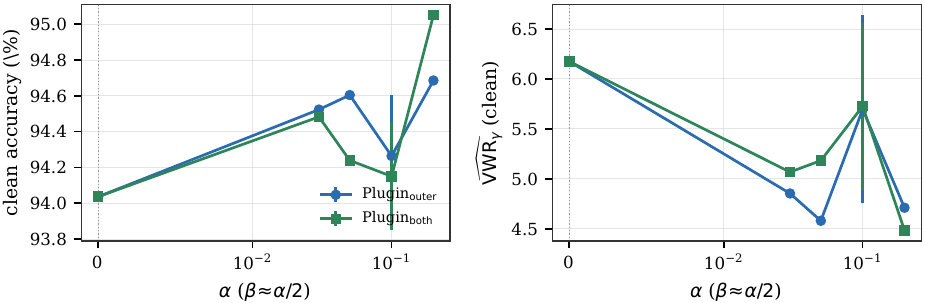}
    % \caption{\textbf{Sensitivity of the plug-in to its strength on CLIP ViT-B/32, Caltech101.}
    % We sweep along the diagonal $\beta \approx \alpha / 2$ for both $\mathrm{Plugin}_{\mathrm{outer}}$
    % and $\mathrm{Plugin}_{\mathrm{both}}$. \emph{Left:} clean accuracy is essentially flat across
    % the whole range $\alpha \in \{0, 0.03, 0.1, 0.3, 1\}$ (maximum deviation $\approx 1$~pp),
    % confirming that the plug-in does not pay a utility tax for its safety effect. \emph{Right:}
    % $\widehat{\mathrm{VWR}}_\gamma$ generally decreases for nonzero plug-in strengths, suggesting that the default $\alpha=0.1$ is not a sharply tuned sweet spot but sits on a flat ``safe-to-increase'' portion of the curve.}
\caption{\textbf{Sensitivity of the plug-in to its strength on CLIP ViT-B/32, Caltech101.}
We sweep along the diagonal $\beta \approx \alpha / 2$ for both $\mathrm{Plugin}_{\mathrm{outer}}$
and $\mathrm{Plugin}_{\mathrm{both}}$. \emph{Left:} clean accuracy remains nearly flat across
$\alpha \in \{0,0.03,0.1,0.3,1\}$, suggesting that the plug-in does not introduce a clear utility
cost over this range. \emph{Right:} the vulnerable-flow measure is generally lower for nonzero
plug-in strengths than for $\alpha=0$, although the response is not strictly monotonic. This suggests
that the default setting is not an isolated tuned point, while also showing that very aggressive
strength choices need not improve the safety metric further.}    \label{fig:alphabeta_sweep}
\end{figure}

\paragraph{Discussion.}
This sweep is intended as a sensitivity check rather than a full hyperparameter search. 
The main observation is that clean accuracy stays stable across the tested range, while the vulnerable-flow measure is usually reduced once the plug-in is activated. 
The response is not strictly monotonic, so we do not claim that larger regularization strength is always better. 
Instead, the sweep supports the weaker and more relevant point that the default plug-in strength is not a fragile single-point choice and that the robustness gains do not come from an obvious clean-accuracy trade-off.

\newpage

\subsection{Additional VLM sanity checks}
\label{app:vlm_sanity_checks}

% Cross-model sanity on Qwen2.5-VL-3B (4-shot, 3 seeds, PGD-10 at eps=1/255).

\begin{table}[h]
\centering
\small
\setlength{\tabcolsep}{3pt}
\caption{\textbf{Cross-model sanity on Qwen2.5-VL-3B-Instruct (4-shot, 3 seeds, PGD-10 at $\varepsilon = 1/255$).} \textbf{Bold = best among LoRA-adv family rows } (same convention as the main VLM table). Clean WC is uninformative on Qwen2.5-VL because the per-class evaluation uses only $\sim 10$ samples per class; PGD WC-Acc is often identically $0$ for the same reason. We therefore regard $\widehat{\mathrm{VSR}}_{\gamma}$ (pgd) and $\mathrm{VWR}_\gamma$ (pgd) as the primary safety signals here.}
\label{tab:qwenvl_crossmodel}
\resizebox{\textwidth}{!}{%
\begin{tabular}{llccccccc}
\toprule
Dataset & Method & Clean Acc $\uparrow$ & PGD Acc $\uparrow$ & Clean WC $\uparrow$ & PGD WC $\uparrow$ & $\widehat{\mathrm{VSR}}_{\gamma}$ (pgd) $\downarrow$ & $\mathrm{VWR}_\gamma$ (pgd) $\downarrow$ & $n$ \\
\midrule
% \multirow{5}{*}{Caltech101}
%  & ZeroShot                 & $99.60$ & $70.80$ & $0.00$ & $0.00$ & $46.17$ & $94.50$ & 1 \\
%  & LoRA                     & $97.73 \pm 0.34$ & $60.40 \pm 1.70$ & $0.00$ & $0.00$ & $31.30 \pm 1.61$ & $107.27 \pm 0.24$ & 3 \\
%  & LoRA-adv                 & $\mathbf{99.40 \pm 0.33}$ & $69.07 \pm 7.17$ & $\mathbf{0.00}$ & $\mathbf{0.00}$ & $\mathbf{21.09 \pm 6.06}$ & $78.95 \pm 18.16$ & 3 \\
%  & LoRA-adv + plugin (outer) & $99.07 \pm 0.57$ & $67.33 \pm 7.72$ & $0.00$ & $0.00$ & $22.80 \pm 4.28$ & $87.49 \pm 13.53$ & 3 \\
%  & LoRA-adv + plugin (both) & $98.80 \pm 1.14$ & $\mathbf{72.33 \pm 7.47}$ & $0.00$ & $0.00$ & $27.11 \pm 12.10$ & $\mathbf{77.69 \pm 22.75}$ & 3 \\
% \midrule
\multirow{5}{*}{DTD}
 & ZeroShot                 & $60.00$ & $3.20$ & $0.00$ & $0.00$ & $19.89$ & $22.25$ & 1 \\
 & LoRA                     & $68.33 \pm 3.90$ & $14.20 \pm 0.82$ & $0.00$ & $0.00$ & $20.63 \pm 3.22$ & $27.75 \pm 0.44$ & 3 \\
 & LoRA-adv                 & $\mathbf{61.47 \pm 3.35}$ & $13.60 \pm 2.12$ & $\mathbf{0.00}$ & $\mathbf{0.00}$ & $10.18 \pm 2.42$ & $18.65 \pm 3.04$ & 3 \\
 & LoRA-adv + plugin (outer) & $61.33 \pm 3.73$ & $\mathbf{14.20 \pm 1.07}$ & $0.00$ & $0.00$ & $\mathbf{8.07 \pm 1.44}$ & $\mathbf{16.15 \pm 2.25}$ & 3 \\
 & LoRA-adv + plugin (both) & $59.20 \pm 2.20$ & $13.40 \pm 1.31$ & $0.00$ & $0.00$ & $9.01 \pm 1.68$ & $17.31 \pm 2.27$ & 3 \\
\bottomrule
\end{tabular}}
\end{table}

\begin{figure}[h]
    \centering
    \includegraphics[width=0.7\textwidth]{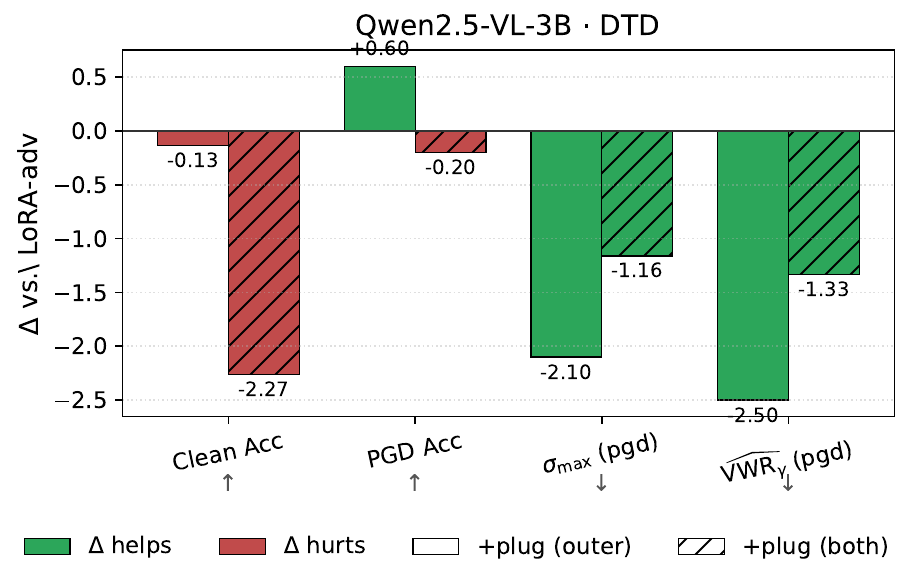}
    \caption{\textbf{QwenVL cross-model summary on DTD.}
    Delta bars of each LoRA-adv $+$ plug-in variant relative to LoRA-adv on the four
    reported metrics. Both variants improve the class-structural safety metrics
    ($\widehat{\mathrm{VSR}}_{\gamma}$, $\widehat{\mathrm{VWR}}_\gamma$) with minimal utility cost, showing that
    the plug-in is not tied to CLIP ViT. 
    % We restrict this visualization to DTD ($K=47$) because on Caltech101 ($K=100$) with 4-shot training, $\widehat{\mathrm{VSR}}_{\gamma}$ is dominated by a
    % few pathologically confusable class pairs (e.g.\ faces / faces\_easy), which decouples
    % it from PGD-Acc and VWR; the full per-cell numbers including Caltech101 remain in
    % Table~\ref{tab:qwenvl_crossmodel}.
    }
    \label{fig:vlm_qwen_delta}
\end{figure}

\begin{table}[h]
\centering
\small
\setlength{\tabcolsep}{3pt}
\caption{\textbf{Weak-attack supplementary results on CLIP ViT-B/32 (4-shot) - smaller radius.} Test robustness is evaluated with PGD at $\varepsilon=0.5/255$ for 100 steps. These results come from the dedicated weak-attack branch and are reported as point estimates from a single seed ($n=1$ for every row). \textbf{Bold = best among LoRA-adv family rows }.}
\label{tab:vlm_vit_weak_eps05}
\resizebox{0.9\textwidth}{!}{%
\begin{tabular}{llccccccc}
\toprule
Dataset & Method & Clean Acc $\uparrow$ & PGD Acc $\uparrow$ & Clean WC $\uparrow$ & PGD WC $\uparrow$ & $\widehat{\mathrm{VSR}}_{\gamma}$ (pgd) $\downarrow$ & $\mathrm{VWR}_\gamma$ (pgd) $\downarrow$ & $n$ \\
\midrule
\multirow{5}{*}{DTD}
 & ZeroShot                       & $42.79$ & $19.92$ & $0.00$ & $0.00$ & $18.38$ & $21.76$ & 1 \\
 % & LoRA                           & $61.47$ & $28.07$ & $27.78$ & $0.00$ & $29.69$ & $26.99$ & 1 \\
 % & LoRA + plugin                  & $61.23$ & $28.25$ & $27.78$ & $0.00$ & $28.69$ & $26.78$ & 1 \\
 & LoRA-adv                       & $59.87$ & $\mathbf{28.43}$ & $\mathbf{16.67}$ & $\mathbf{0.00}$ & $27.73$ & $26.17$ & 1 \\
 & LoRA-adv + plugin (inner)      & $\mathbf{60.22}$ & $\mathbf{28.43}$ & $\mathbf{16.67}$ & $\mathbf{0.00}$ & $27.96$ & $26.02$ & 1 \\
 & LoRA-adv + plugin (outer)      & $59.93$ & $28.07$ & $\mathbf{16.67}$ & $\mathbf{0.00}$ & $\mathbf{27.33}$ & $25.72$ & 1 \\
 & LoRA-adv + plugin (both)       & $59.87$ & $28.01$ & $\mathbf{16.67}$ & $\mathbf{0.00}$ & $27.63$ & $\mathbf{25.50}$ & 1 \\
\midrule
\multirow{5}{*}{OxfordPets}
 & ZeroShot                       & $85.04$ & $36.74$ & $0.00$ & $0.00$ & $46.69$ & $51.86$ & 1 \\
 & LoRA-adv                       & $86.18$ & $34.94$ & $51.00$ & $4.00$ & $49.36$ & $64.42$ & 1 \\
 & LoRA-adv + plugin (inner)      & $86.21$ & $34.86$ & $51.00$ & $4.00$ & $49.18$ & $63.93$ & 1 \\
 & LoRA-adv + plugin (outer)      & $\mathbf{87.33}$ & $35.68$ & $\mathbf{56.00}$ & $\mathbf{5.00}$ & $46.13$ & $62.30$ & 1 \\
 & LoRA-adv + plugin (both)       & $87.22$ & $\mathbf{35.90}$ & $\mathbf{56.00}$ & $\mathbf{5.00}$ & $\mathbf{45.88}$ & $\mathbf{61.74}$ & 1 \\
\midrule
\multirow{5}{*}{Caltech101}
 & ZeroShot                       & $91.40$ & $66.65$ & $6.67$ & $0.00$ & $15.78$ & $38.60$ & 1 \\
 & LoRA-adv                       & $94.04$ & $72.66$ & $26.67$ & $\mathbf{0.00}$ & $15.32$ & $29.87$ & 1 \\
 & LoRA-adv + plugin (inner)      & $94.00$ & $\mathbf{72.74}$ & $\mathbf{33.33}$ & $\mathbf{0.00}$ & $15.51$ & $30.24$ & 1 \\
 & LoRA-adv + plugin (outer)      & $\mathbf{94.16}$ & $72.70$ & $\mathbf{33.33}$ & $\mathbf{0.00}$ & $15.29$ & $29.84$ & 1 \\
 & LoRA-adv + plugin (both)       & $94.08$ & $72.62$ & $26.67$ & $\mathbf{0.00}$ & $\mathbf{15.22}$ & $\mathbf{29.74}$ & 1 \\
\bottomrule
\end{tabular}}
\end{table}

\newpage
\begin{table}[h]
\centering
\small
\setlength{\tabcolsep}{3pt}
\caption{\textbf{Weak-attack supplementary results on CLIP ViT-B/32 (4-shot) - fewer steps.} Test robustness is evaluated with PGD at $\varepsilon=1.0/255$ for 20 steps. These results come from the dedicated weak-attack branch and are reported as point estimates from a single seed ($n=1$ for every row). Caltech101 was not included in this branch. \textbf{Bold = best among LoRA-adv family rows }.}
\label{tab:vlm_vit_weak_eps10}
\resizebox{0.9\textwidth}{!}{%
\begin{tabular}{llccccccc}
\toprule
Dataset & Method & Clean Acc $\uparrow$ & PGD Acc $\uparrow$ & Clean WC $\uparrow$ & PGD WC $\uparrow$ & $\widehat{\mathrm{VSR}}_{\gamma}$ (pgd) $\downarrow$ & $\mathrm{VWR}_\gamma$ (pgd) $\downarrow$ & $n$ \\
\midrule
\multirow{5}{*}{DTD}
 & ZeroShot                       & $42.79$ & $14.72$ & $0.00$ & $0.00$ & $19.27$ & $21.78$ & 1 \\
 % & LoRA                           & $61.47$ & $20.80$ & $27.78$ & $0.00$ & $32.24$ & $27.32$ & 1 \\
 % & LoRA + plugin                  & $61.17$ & $20.45$ & $27.78$ & $0.00$ & $31.22$ & $27.31$ & 1 \\
 & LoRA-adv                       & $60.05$ & $21.04$ & $\mathbf{16.67}$ & $\mathbf{0.00}$ & $29.78$ & $26.83$ & 1 \\
 & LoRA-adv + plugin (inner)      & $\mathbf{60.22}$ & $\mathbf{21.22}$ & $\mathbf{16.67}$ & $\mathbf{0.00}$ & $29.91$ & $26.76$ & 1 \\
 & LoRA-adv + plugin (outer)      & $60.05$ & $21.10$ & $\mathbf{16.67}$ & $\mathbf{0.00}$ & $29.63$ & $26.30$ & 1 \\
 & LoRA-adv + plugin (both)       & $60.28$ & $20.98$ & $\mathbf{16.67}$ & $\mathbf{0.00}$ & $\mathbf{29.05}$ & $\mathbf{26.17}$ & 1 \\
\midrule
\multirow{5}{*}{OxfordPets}
 & ZeroShot                       & $85.04$ & $21.45$ & $0.00$ & $0.00$ & $47.68$ & $50.80$ & 1 \\
 & LoRA-adv                       & $86.18$ & $19.41$ & $51.00$ & $1.00$ & $56.04$ & $62.25$ & 1 \\
 & LoRA-adv + plugin (inner)      & $86.21$ & $19.54$ & $49.00$ & $\mathbf{2.00}$ & $55.82$ & $62.04$ & 1 \\
 & LoRA-adv + plugin (outer)      & $\mathbf{87.27}$ & $\mathbf{20.55}$ & $\mathbf{56.00}$ & $1.01$ & $\mathbf{53.37}$ & $60.80$ & 1 \\
 & LoRA-adv + plugin (both)       & $86.70$ & $20.33$ & $47.00$ & $\mathbf{2.00}$ & $53.47$ & $\mathbf{60.71}$ & 1 \\
\bottomrule
\end{tabular}}
\end{table}

\paragraph{Discussion.}
These tables are intended as supplementary sanity checks rather than primary evidence.
Table~\ref{tab:qwenvl_crossmodel} shows that the plug-in can also reduce vulnerable-flow measures on Qwen2.5-VL, suggesting that the effect is not tied only to CLIP ViT-B/32.
The weak-attack CLIP results in Tables~\ref{tab:vlm_vit_weak_eps05} and~\ref{tab:vlm_vit_weak_eps10} are single-seed point estimates, so we interpret them conservatively.
Across these weaker PGD protocols, the outer and both variants usually reduce the vulnerable-flow measures with little change in clean accuracy, while inner-only placement is less stable.
PGD worst-class accuracy remains floor-limited in several cells, so these tables are used mainly to check that the vulnerable-flow compression persists under alternative VLM evaluation protocols.

\section{Compute resources}
\label{app:compute_resources}
All experiments are conducted in Python\textsuperscript{\textregistered} on a machine equipped with an AMD EPYC\textsuperscript{\textregistered} 7452 32-Core Processor, 128GB of RAM, and one A100 GPU with 40GB of VRAM.

% \clearpage
% \input{checklist}

%%%%%%%%%%%%%%%%%%%%%%%%%%%%%%%%%%%%%%%%%%%%%%%%%%%%%%%%%%%%

% \newpage
% \input{checklist.tex}

\end{document}